\documentclass[nohyperref]{article}

\usepackage{microtype}
\usepackage{graphicx}
\usepackage{subfigure}
\usepackage{booktabs} 

\usepackage{hyperref}

\usepackage[accepted]{icml2022}

\usepackage{amsmath}
\usepackage{amssymb}
\usepackage{mathtools}
\usepackage{amsthm}

\usepackage[capitalize,noabbrev]{cleveref}

\theoremstyle{plain}
\newtheorem{theorem}{Theorem}[section]
\newtheorem{proposition}[theorem]{Proposition}
\newtheorem{lemma}[theorem]{Lemma}

\theoremstyle{definition}

\theoremstyle{remark}
\newtheorem{remark}[theorem]{Remark}

\usepackage[textsize=tiny]{todonotes}

\icmltitlerunning{Informative labels in SSL}

\usepackage{enumerate}
\usepackage{enumitem}
\newcommand\independent{\protect\mathpalette{\protect\independenT}{\perp}}
\def\independenT#1#2{\mathrel{\rlap{$#1#2$}\mkern2mu{#1#2}}}

\usepackage{graphicx,tikz}
\graphicspath{{./img/}}
\usetikzlibrary{bayesnet}

\usepackage{amssymb}
\usepackage{bbold}
\usepackage{amsmath}

\usepackage{algorithm}
\usepackage{algorithmic}

\begin{document}

\twocolumn[
\icmltitle{Are labels informative in semi-supervised learning? \\ Estimating and leveraging the missing-data mechanism}



\icmlsetsymbol{equal}{*}

\begin{icmlauthorlist}
\icmlauthor{Aude, Sportisse}{yyy,comp}
\icmlauthor{Hugo Schmutz}{yyy,comp,sch,sch1}
\icmlauthor{Olivier Humbert}{sch,yy}
\icmlauthor{Charles Bouveyron}{yyy,comp}
\icmlauthor{Pierre-Alexandre Mattei}{yyy,comp}
\end{icmlauthorlist}

\icmlaffiliation{yyy}{3iA Côte d'Azur}
\icmlaffiliation{comp}{Centre Inria d'Université Côte d'Azur, Equipe Massai}
\icmlaffiliation{sch1}{Université Côte d'Azur, Laboratoire J.A. Dieudonné, UMR CNRS 7351}
\icmlaffiliation{sch}{Université Côte d'Azur, TIRO-MATOS, UMR CEA E4320}
\icmlaffiliation{yy}{Centre Antoine Lacassagne}
\icmlcorrespondingauthor{Aude Sportisse}{aude.sportisse@inria.fr}

\icmlkeywords{Machine Learning, ICML}

\vskip 0.3in
]




\begin{abstract}
  Semi-supervised learning is a powerful technique for leveraging unlabeled data to improve machine learning models, but it can be affected by the presence of  ``informative" labels, which occur when some classes are more likely to be labeled than others. In the missing data literature, such labels are called missing not at random. In this paper, we propose a novel approach to address this issue by estimating the missing-data mechanism and using inverse propensity weighting to debias any SSL algorithm, including those using data augmentation. We also propose a likelihood ratio test to assess whether or not labels are indeed informative. Finally, we demonstrate the performance of the proposed methods on different datasets, in particular on two medical datasets for which we design pseudo-realistic missing data scenarios.
\end{abstract}

\section{Introduction}

Technological advancements have enabled the collection and storage of vast amounts of data, offering real hope for better prediction of phenomena. Unfortunately, this also leads to dirty data, and more specifically, missing data. In this paper, we focus on the scenario where a large amount of data is available, but labeling the data is costly, time-consuming, or even risky (for instance, medical data collection requires invasive tests for patients). Semi-supervised learning (SSL) \citep{chapelle2009semi,van2020survey} has emerged as a crucial problem to leverage both labeled and unlabeled data in predictive models. The unlabeled data are treated as observations with missing labels, as previously done in various studies \citep{grandvalet2004semi,ahfock2019missing,hu2022on,schmutz2022don}. Recently, SSL algorithms have been extended to deep learning techniques demonstrating remarkable empirical successes, particularly through the systematic use of data augmentation \citep{berthelot2019remixmatch,xie2020unsupervised,sohn2020fixmatch,rizve2021defense}.

One of the challenges in semi-supervised learning is that the distribution of labels in the unlabeled dataset is unknown. For instance, it is uncertain whether a class that is well-represented in the labeled images is also well-represented in the unlabeled images. The traditional approach is to assume that the label distributions are identical in the labeled and unlabeled datasets. This assumption implies that people label classes in equal proportions, regardless of the class nature or the quality of the images. However, it disregards the potential unbalance of popularity among classes. 
For example, in a medical context, doctors may prioritize labeling the class of sick patients or leave unlabeled the data with an ambiguous diagnosis. When the label distribution differs in the labeled and unlabeled datasets, the missing labels are said to be informative or Missing Not At Random (MNAR). The missingness of a label must be taken into account
to obtain results from the available data that can be generalized to the entire population \citep{rubin1976inference}. 
This is usually modeled by the missing-data mechanism, i.e. the probability of a sample to be observed (depending on the values of the label itself). Recently, it has been shown that classical SSL algorithms indeed fail to provide accurate results for the less observed classes in presence of informative labels. As there is a selection bias in the sample, MNAR data also raise the issue that some models can lead to non-identifiable parameters \citep{baker1988regression,miao2016identifiability}. A major challenge is that testing whether the data is indeed MNAR is difficult \citep{d2010new}, but it is necessary to provide a guideline for choosing which algorithm to apply. The main objective of this paper is to address these issues by estimating the missing-data mechanism. 

Beyond the scope of deep learning, in the missing-data literature, significant works have considered the case of MNAR responses \citep{tang2018statistical}. \citet{ibrahim1996parameter} and \citet{ibrahim2001missing} estimate the parameters of both the model and the missing-data mechanism using the Expectation-Maximization (EM) algorithm in binomial regression and generalized linear models. They also propose a likelihood ratio test statistic for selecting the variables related to the missingness but leave the identifiability of the parameters in perspective. It is in the semi-parametric setting that a lot of work has been done to obtain identification results, most often using a \textit{shadow} variable \citep{miao2015identification,miao2016varieties} that adds auxiliary information \citep{molenberghs2008every}. Some works \citep{shao2016semiparametric,morikawa2017semiparametric} also propose to debias classical estimators by using inverse probability
weighting (IPW) techniques, weighting each sample by the inverse of its probability of being observed
as determined by the missing-data mechanism. However, only the recent work of \citet{hu2022on}  proposes an extension to deep learning, debiasing the risk estimator with a propensity score, but they do not directly model the missing-data mechanism, which is the main focus of our study (see Section \ref{sec:compar} for a comprehensive comparison).



\looseness=-1
Our key contributions are summarized as follows: 
\begin{itemize}[topsep=0pt,itemsep=0pt,leftmargin=*]
    \item We consider a general self-masked MNAR model and prove its identifiability, showing in the process of identifiability of the model of \citep{hu2022on}.
    \item We propose two estimates of the missing-data mechanism and show their consistency. 
    \item Based on these estimators, we propose an algorithm using IPW techniques able to debias any SSL algorithm in presence of informative labels.
    \item We provide a heuristic procedure to test whether the labels are indeed MNAR. 
    \item We first demonstrate the efficiency of our methods on
    classical datasets. {Furthermore}, we propose two pseudo-
    realistic MNAR scenarios using medMNIST datasets \citep{yang2021medmnist}. These contrast with the toy missing-data scenarios often used in existing works, even when the method is designed to handle informative labels. 
\end{itemize}


\section{Informative labels}
\label{sec:informativelabels}

\subsection{Missing labels typology}

In this paper, we study a dataset of $n$ samples, denoted as $D = {(x_i, y_i)}_{i=1}^n$, where $x_i$ represents the features and $y_i$ represents the labels, which are drawn from the distribution $p(x,y)=p(x)p(y|x)$. Some of the labels are supposed to be missing, and thus the dataset is split into two subsets: a labeled dataset $D_{\ell} = {(x_i, y_i)}_{i=1}^{n_\ell}$ of size $n_\ell$ and an unlabeled dataset $D_{u} = {(x_i)}_{i=n_\ell+1}^{n}$ of size $n_u = n - n_\ell$. The distribution of the labeled dataset is denoted as $p^\ell(.)$ (resp. $p^u(.)$ for the unlabeled dataset). In the following, we consider a discrete set of labels {denoted as $\mathcal{C}=\{1,\dots,K\}$, with $K$ the  number of classes}.

Most of the semi-supervised learning methods make the following assumption:
\begin{enumerate}[label=\textbf{A\arabic*.},topsep=0pt]
    \item \label{ass:SSL2} The marginal distributions of the features and of the labels are identical in the labeled and unlabeled dataset, i.e.\ $p^\ell(x)=p^u(x), \: \forall x$ and  $p^\ell(y)=p^u(y), \: \forall y$. 
\end{enumerate}
Assumption \ref{ass:SSL2} means that people label classes in equal proportions, regardless of their nature (label) or the quality of the images (features).
Modeling two separate distributions, $p^\ell(.)$ and $p^u(.)$, is not always convenient, so we instead use a notation commonly used in missing-data studies. We introduce an additional random variable called missing-data indicator, $r\in \{0,1\}$, where $r = 1$ if $y$ is observed and $r = 0$ if $y$ is missing. For example, this notation implies $p^\ell(x) = p(x|r = 1)$ and $p^u(x) = p(x|r = 0)$. According to Rubin's \citeyearpar{rubin1976inference} typology, labels can be:
(i) Missing Completely At Random (MCAR) if the cause of the missingness is completely independent from the data values, i.e. $r \independent x,y$ (equivalent to Assumption \ref{ass:SSL2}), (ii) Missing At Random (MAR) if the cause of the missingness can be explained by the features, $r \independent y | x$ and (iii) Missing Not At Random (MNAR) in all other cases.
 For example, labels will be MAR if medical doctors are less likely to label analyses that are of poor quality but they will be MNAR if they prefer to label the class of sick patients first. This last situation creates an ``unbalanced class popularity" for which Assumption \ref{ass:SSL2} is no longer valid.

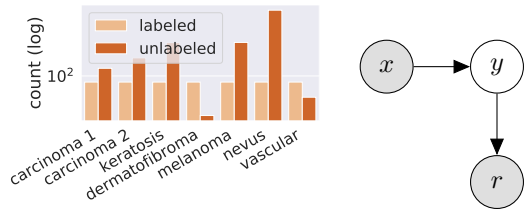
\begin{figure}
\centering
\includegraphics[width=0.25\textwidth]{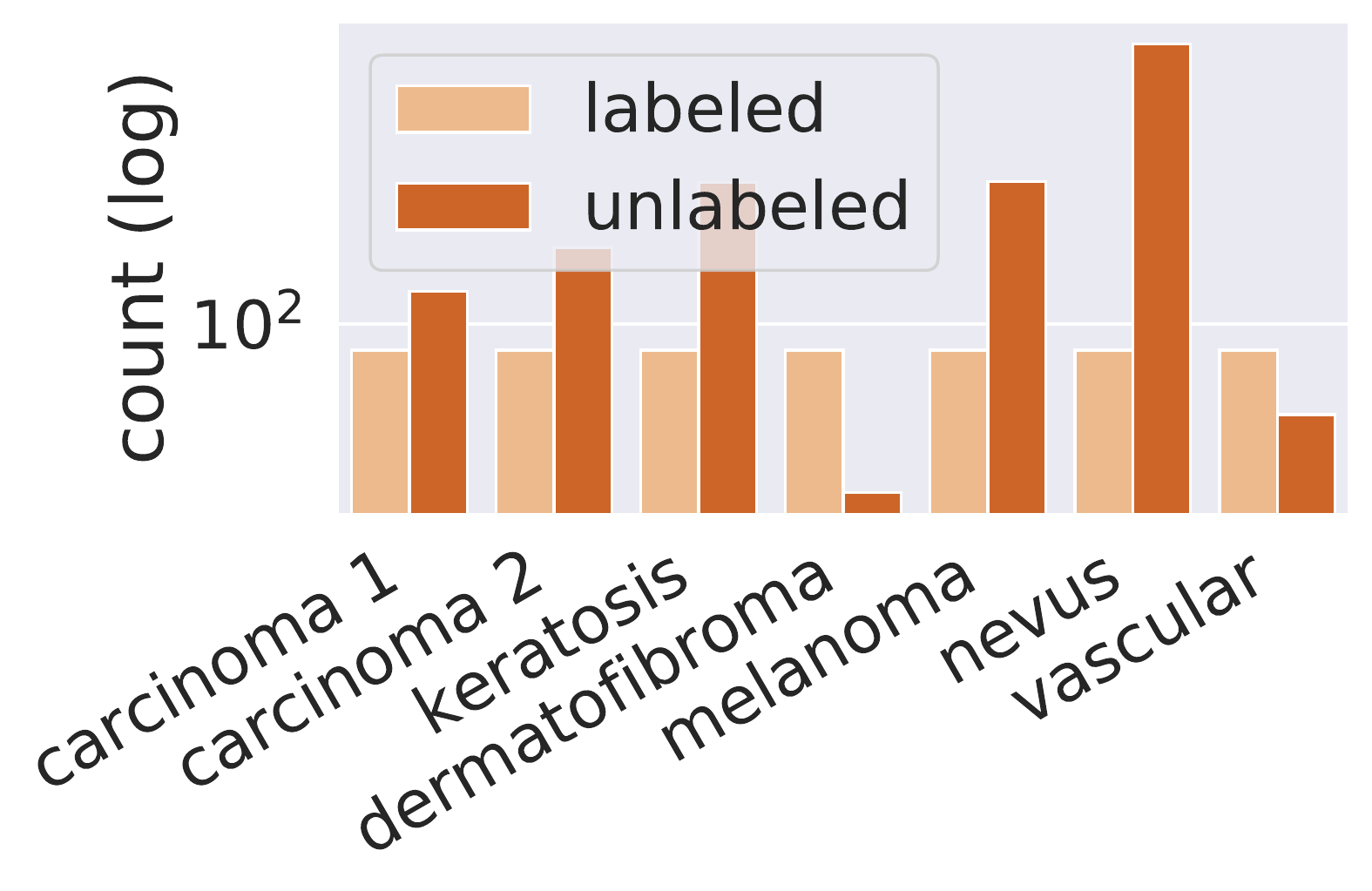}
\hskip 0.15in
    \begin{tikzpicture}
        \node[obs]                               (y1) {$x$};
    	\node[latent, right=0.75cm of y1]            (y2) {$y$};
        \node[obs,below of=y2, yshift=-1.5cm]                               (r2) {$r$};
    		
    	\draw[->] (y1) edge (y2);
    	\draw[->] (y2) edge (r2);
    \end{tikzpicture}
         \caption{ Left panel: MNAR labels for dermaMNIST (log count of labeled and unlabeled images per class). Right panel: Structural causal graph of the self-masked mechanism (Assumption \ref{ass:ourSSL2}). The nodes in grey represent fully observed variables and the edges from $x$ to $y$ means that $x$ causes $y$.} 
         \label{fig:graphical}
\end{figure}

In this paper, we focus on the MNAR case, specifically, the labels are assumed to be ``self-masked MNAR", meaning their unavailability only depends on their own values. This assumption allows us to model the unbalanced class popularity situation (see Figure \ref{fig:graphical}) and is widely used in the missing-data literature \citep{mohan2018handling,sportisse2020estimation}. Our assumption is formalized as follows:
\begin{enumerate}[label=\textbf{A\arabic*.},topsep=0pt]

    \setcounter{enumi}{1}
    \item \label{ass:ourSSL2} The labels are self-masked MNAR, i.e.\ $r \independent x | y$. 
\end{enumerate}

Assumption \ref{ass:ourSSL2} is weaker than Assumption \ref{ass:SSL2}, since
the equality of marginal distributions, either for features or labels, in both the labeled and unlabeled datasets, is relaxed. It only requires that the conditional distribution of the features given the class is the same in the labeled and unlabeled datasets (i.e. $p^\ell(x|y)=p^u(x|y), : \forall x,y$). For example, it does not cover the case where the radiography of sick patients does not have the same resolution whether it is labeled or not.

\begin{remark}[More general assumptions]
The notation introduced here does not allow to consider different sets for the labels present in the labeled dataset and in the unlabeled dataset. Label distribution mismatch has already been considered, such as when new classes appear in the unlabeled dataset \citep{guo2020safe, cao2021open} or when none of the classes present in the labeled dataset are present in the unlabeled dataset \citep{chen2020semi, huang2021universal}. However, these works are beyond the context of our work, as they do not allow to directly account for the informative nature of the labels, which is the main focus of this paper.




\end{remark}

\subsection{Non-ignorable missing-data mechanism}
\label{sec:notign}

This typology of missing-data mechanism is important in determining the appropriate method to use: statistical inference can be performed on $p(x,y)$ for MCAR or MAR labels, but it should be performed on $p(x,y,r)$ for MNAR labels \citep{little2019statistical}. We denote the parameter of interest, $\theta \in \Theta$, of $p(y|x;\theta)$. {In most cases, this parameter corresponds to the weights of a neural network, or in simpler cases, a logistic regression. The parameter $\phi$ of the missing data mechanism $p(r|x,y;\phi)$ lives in $\Phi=[0,1]^K$.} In the following, we assume that the parameters are distinct, meaning that the joint parameter space is equal to $\Theta \times \Phi$. Following the common notation introduced by \citet{le2020linear}, the observed label vector is $(y \odot r)$, where $\odot$ represents {the term-by-term product}, such that $(y \odot r)_i = y_i$ if $r_i = 1$ and $(y \odot r)_i = \mathrm{NA}$ if $r_i = 0$.

The traditional method for estimating $\theta$ is to minimize the negative observed log-likelihood:
{\small
\begin{align}
\nonumber
&\ell(\theta,\phi)=-\sum_{i=1}^{n} \log p(x_i,y_i\odot r_i,r_i;\theta,\phi)
\\
\nonumber
&=-\sum_{i=1}^n \log p(r_i|x_i,y_i\odot r_i;\phi)p(y_i\odot r_i|x_i;\theta)p(x_i) \\
&= -\sum_{i=1}^n\left\{
    \begin{array}{ll}
        \log p(r_i|x_i,y_i;\phi)p(y_i|x_i;\theta) & \mbox{if } r_i=1 \\
        \log\sum_{\tilde{y}\in \mathcal{C}} p(r_i|x_i,\tilde{y};\phi)p(\tilde{y}|x_i;\theta) & \mbox{if } r_i=0
    \end{array} \right. + C
\label{eq:loglikeli}    
    \\ \nonumber
&= -\sum_{i=1}^n r_i\log p(y_i|x_i;\theta) + C \: \: \textrm{under M(C)AR assumption},
\end{align}}
where $C$ is a constant independent of $\theta$. 
For M(C)AR labels, we use in the last step that $p(r_i|x_i,{y};\phi)$ does not depend on ${y}$ and that $\sum_{\tilde{y}\in \mathcal{C}}p(\tilde{y}|x_i;\theta)=1$; the result follows if $\phi$ is considered as a nuisance term. {This simple calculation is a common technique in the missing data literature \citep{little2019statistical}.} It implies that for MCAR or MAR labels, it is not necessary to estimate the missing-data mechanism and minimizing the complete likelihood (on labeled data only) is sufficient. However, for MNAR labels, the missing-data mechanism cannot be ignored and must be taken into account.

\subsection{Identification of the joint distribution}

To fix ideas, Figure \ref{fig:graphical} (right panel) shows the causal relationships between the variables $x, y$ and the missing-data indicator $r$ through a structural causal graph \citep{neuberg2003causality}. Assumption \ref{ass:ourSSL2} allows us to get the nonparametric identification of the joint distribution $p(y,x,r)$, i.e.\ it can be expressed with quantities involving only observed data. Specifically, in the self-masked setting, the features act as shadow variables, providing enough auxiliary observed information to achieve the identifiability of the parameters.




\begin{proposition}[Identification of the joint distribution]
Under Assumptions \ref{ass:ourSSL2} (self-masked MNAR), the joint distribution $p(y,x,r)$ is identified. 
\end{proposition}

This result is a corollary of Theorem 1 in \citep{miao2015identification} 
 and is proved in Appendix \ref{sec:proof_id}. It is worth noting that this result also demonstrates the identification in \citep{hu2022on}.

\section{Debiasing classical SSL algorithms}\label{sec:debiasSSL}

\subsection{Complete-case: learning with labeled data}\label{sec:labelCC}

In classical supervised learning, the aim is to learn a predictive model $p(y|x;\theta)$, parametrized by $\theta \in \Theta$, by minimizing the theoretical risk:
$$\theta^\star=\mathrm{argmin}_{\theta\in \Theta} \quad \mathcal{R}(\theta):=\mathbb{E}_{(x,y)\sim p(x,y)}[\ell_\ell(\theta;x,y)],$$
where $\ell_\ell$ is typically the negative log-likelihood function $\ell_\ell(\theta;x_i,y_i)=-\log p(y_i|x_i;\theta)$ but can be any loss function. The theoretical risk is never observed, as it requires the knowledge of the true distribution $p(x,y)$. A typical learning procedure is then the minimization of the empirical risk, which is an unbiased estimate of the theoretical risk:
$$\hat{\theta}=\mathrm{argmin}_{\theta\in \Theta} \quad \hat{\mathcal{R}}(\theta):= \frac{1}{n}\sum_{i=1}^n \ell_\ell(\theta;x_i,y_i).$$

This quantity is known in the supervised learning setting, but is still unobserved in the presence of missing labels. For MCAR labels, the natural estimator for $\mathcal{R}(\theta)$ is the complete-case empirical risk, computed for the labeled data only: $\hat{\mathcal{R}}^{CC}(\theta):=\frac{1}{n_\ell}\sum_{i=1}^n r_i \ell_\ell(\theta;x_i,y_i)$. It is unbiased for MCAR labels but not in the other cases. For MAR labels, \citet{liu2020kernel} propose to use the IPW estimator defined as $\tilde{\mathcal{R}}^{\mathrm{IPW,MAR}}(\theta):=\frac{1}{n}\sum_{i=1}^n \frac{r_i \ell_\ell(\theta;x_i,y_i)}{\pi^{\mathrm{MAR}}(x_i)}$, where $\pi^{\mathrm{MAR}}(x)=\mathbb{P}(r=1|x)$. Similarly, for self-masked MNAR labels, we propose the following IPW estimator: 
\begin{equation}\label{eq:IPWestimatorMNAR}
{\hat{\mathcal{R}}}_\phi(\theta):=\frac{1}{n}\sum_{i=1}^n \frac{r_i \ell_\ell(\theta;x_i,y_i)}{\phi_{y_i}},
\end{equation}
where $\phi=(\phi_0, ..., \phi_K)\in\Phi=[0, 1]^K$, and  $\forall k \in \mathcal{C}, \: \phi_k := \mathbb{P}(r=1|y=k)$. 

The idea behind the IPW technique is that one labeled sample $(x_i,y_i,r_i=1)$ should not be counted only once but should take into account that there are unlabeled samples $(x_j,y_j,r_j=0), j\neq i$ that belong to the same class ($y_j=y_i$). As a result, it is then counted $1/\phi_{y_i}$ times. For example, if the probability of being observed in a class is one third, an observed sample from that class will be counted three times.

\begin{proposition}[Unbiasedness of the IPW estimator]\label{prop:IPWunbiased}
    The IPW estimator proposed in \eqref{eq:IPWestimatorMNAR} is unbiased, if the mechanism is well specified, i.e. $\mathbb{E}[r|y]=\phi_y$.
\end{proposition}
\begin{proof}
\begin{align*}
    \mathbb{E}\left[ \frac{r}{\phi_y}\ell_\ell(\theta;x,y)\right]
    &=\mathbb{E}\left[\mathbb{E}\left[ \frac{r}{\phi_y}\ell_\ell(\theta;x,y)|y\right]\right] & \\
    &=\mathbb{E}\left[\mathbb{E}\left[ \frac{r}{\phi_y}|y\right]\mathbb{E}\left[ \ell_\ell(\theta;x,y)|y\right]\right]  &
    \\&=\mathbb{E}\left[\frac{\mathbb{E}[r|y]}{\phi_y}\mathbb{E}\left[ \ell_\ell(\theta;x,y)|y\right]\right]
    =\mathcal{R}(\theta),
\end{align*}
using $r\independent x | y$ in the second equality. 
\end{proof}

\subsection{Incorporating the unlabeled data}\label{sec:compar}

A major drawback of the classical IPW estimator in \eqref{eq:IPWestimatorMNAR} is that it only uses labeled data and not all available data. To address this, traditional SSL algorithms for MCAR labels add a regularization term to the classical supervised objective:
{\small
\begin{equation}\label{eq:estimatorclassical}
\hat{\mathcal{R}}^\textrm{SSL}(\theta):=\frac{1}{n} \sum_{i=1}^{n} 
r_i \frac{\ell_\ell(\theta;x_i,y_i)}{n_\ell/n} + \frac{\lambda}{n} \sum_{i=1}^{n}(1-r_i)\frac{\ell_u(\theta;x_i)}{n_u/n},
\end{equation}
}

where $\lambda>0$ is a regularization term. The function $\ell_u$ {is a loss function which does not depend on the labels}; \citet{schmutz2022don} note that $\ell_u$ can be viewed in many cases as a surrogate of $\ell_\ell$. 

For example, \citet{grandvalet2004semi} use the Shannon entropy. 
Another popular approach is to use ``pseudo-labels" \citep{rizve2021defense} for unlabeled data by selecting the class with the highest posterior probability.
Only the pseudo-labels that have a predicted probability higher than a predefined threshold $\tau$ are used as targets.
Both methods encourage the model to have a high level of confidence when imputing unlabeled data, but pseudo-label methods only use data points that have already been predicted with high confidence. Recently, state-of-the-art methods such as \citep{sohn2020fixmatch,berthelot2019remixmatch} have also been developed to make the model more robust to data augmentation of the features.

For MNAR labels, the standard estimator in \eqref{eq:estimatorclassical} is biased, and we propose the following estimator, which is unbiased if the mechanism is correctly specified, using the same argument as Proposition \ref{prop:IPWunbiased}:
\begin{multline}
\label{eq:drclassical}
    \hat{\mathcal{R}}^{\mathrm{SSL}}_\phi(\theta):=\frac{1}{n}\sum_{i=1}^n \frac{r_i \ell_\ell(\theta;x_i,y_i)}{{\phi}_{y_i}} \\
    -\frac{\lambda}{n}\sum_{i=1}^n \frac{r_i-{\phi}_{y_i}}{{\phi}_{y_i}}\ell_u(\theta;x_i).
\end{multline}
This estimator has the significant advantage of being able to debias any SSL algorithm, including methods using data augmentation, with the knowledge of the weights $\phi_{y_i}$.  This is the MNAR counterpart of the estimator proposed by \citet{schmutz2022don} for MCAR data {and of} the one suggested by \citet{liu2020kernel} for MAR data with $\lambda=1$. The only difference is the form of the mechanism: for MCAR data, ${\phi}_{y_i}=n_\ell/n, \forall i$ and for MAR data, they use $\pi^{\mathrm{MAR}}(x)$ defined in Section \ref{sec:labelCC} instead.


\paragraph{Comparison with the work of \citet{hu2022on}}
{Our estimator also shares similarities with the ``doubly-robust" estimate suggested by \citet{hu2022on},} for debiasing the classical SSL estimators of the risk for self-masked MNAR labels (Assumption \ref{ass:ourSSL2}). They build their estimator upon a very interesting strategy also used in the missing-data literature (see the recent review of \citet{rabe2022ignoring}) and leverage from it to indirectly account for the informative nature of the missing labels. The form of the risk estimator is as \eqref{eq:drclassical}, but they target a composite likelihood (which does not encompass the cross entropy) by starting from $\hat{\theta} \in \mathrm{argmin}_\theta -\log p(x|y;\theta)$ instead of $\hat{\theta} \in \mathrm{argmin}_\theta -\log p(y|x;\theta)$. 
The biggest advantage of their strategy is that it allows not to directly model the missing-data mechanism, which can be tedious in some missing data scenarios. Besides, the fundamental difference between our work and theirs is their weight does not involve the missing-
data mechanism, but only the class proportions $p(y)$. Their method thus encourages the model to be accurate for the least frequent classes (when $p(y)$ is small) but will not detect or favour the least labeled classes (when $\phi_y$ is small). On the contrary, our method will benefit from the estimation of the missing-data mechanism, which can be obtained at no extra computational cost (see Section \ref{sec:alg}). 

\section{Estimating the missing-data mechanism}
\label{sec:estim_meca}

In Section \ref{sec:debiasSSL}, we proposed two unbiased estimates of the risk when the labels are informative. However, both require the knowledge of the missing-data mechanism. In practice, we will use the estimators given in \eqref{eq:IPWestimatorMNAR} and \eqref{eq:drclassical} by plugging-in an estimation $\hat{\phi}$ of the mechanism, resulting in $\hat{\mathcal{R}}_{\hat{\phi}}$ and $\hat{{\mathcal{R}}}^{\mathrm{SSL}}_{\hat{\phi}}(\theta)$.
In this section, we provide two estimators of the missing-data mechanism by using either the method of moments or the method of maximum likelihood.
 

\subsection{Moment estimator}

A possible estimator of the missing-data mechanism is obtained by the method of moments applied to \mbox{$p(r=1,y=y)=\mathbb{E}\left[\mathbb{1}_{\{r=1,y\}}\right]=\phi_y p(y)$}. It implies
\begin{equation}\label{eq:meca_mm_oracle}
{\phi}^{M}_y=\frac{\sum_{i=1}^n \mathbb{1}_{\{r=1,y_i=y\}}}{n}\frac{1}{p(y)}, \forall y \in \mathcal{C}.
\end{equation}
This estimator allows us to leverage the information we have on the labeled data, because $p(r=1,y)$ is estimated by counting the number of labeled data in each class ($\mathbb{1}_{{r=1,y_i=y}}$). The challenge now is to estimate the class distribution $p(y)$. This allows for two simple cases where the mechanism can be calculated directly: (i) when we know that the entire dataset is balanced (use $p(y)=1/K$) and (ii) when we have prior information on the class proportions (use $p(y)=p_{\textrm{prior}}(y)$). This last case can happen when we have data from the general population (e.g. we know the prevalence rate of a disease).

When the class proportions are unknown, we propose to estimate $p(y)$ as follows:
\begin{equation}\label{eq:py}
\hat{p}(y;\theta)=\frac{1}{n}\sum_{i=1}^{n}p(y_i|x_i;\theta), \forall \theta \in \Theta,
\end{equation}
which is a consistent estimator of $p(y)$ noting that $p(y)=\int p(y|x;\theta)p(x)dx$.
The estimator of the missing-data mechanism is thus: 
\begin{equation}\label{eq:estim_moment}
(\hat{\phi}_y^{M})_{\theta}=\frac{\sum_{i=1}^n \mathbb{1}_{\{r=1,y_i=y\}}}{n}\frac{1}{\hat{p}(y,\theta)}, \forall y \in \mathcal{C}.
\end{equation}

\begin{remark}[Computation of $p_\theta(y)$]
The estimator of the class proportions defined in \eqref{eq:py} cannot be incorporated as such in a SGD algorithm, typically used to estimate $\theta$. We propose two ways to compute it within a mini-batch:
\begin{itemize}[topsep=0pt,itemsep=1pt,leftmargin=*]
\item use a moving averaging strategy inspired by \citep{hu2022on}, by using a buffer $\hat{p}_{\textrm{buffer}}(y)$ updated at each iteration with $\hat{p}(y,\theta=\theta_b)$ \eqref{eq:py}, where $\theta_b$ is the parameter of the current mini-batch $b$\footnote{In this method, while $(\hat{\phi}_y^{M})_{\theta}$ depends on $\theta$, we do not propagate the gradients through $\theta$.}: 
\begin{equation}\label{eq:buffer}
\hat{p}_{\textrm{buffer}}(y)=\mu \hat{p}_{\textrm{buffer}}(y) + (1-\mu) \hat{p}(y,\theta_b).
\end{equation}
\item propagate the gradients through $(\hat{\phi}_y^{M})_{\theta}$.
\end{itemize}
\end{remark}

\subsection{Maximum likelihood estimator}\label{sec:maximum_ll}

The second estimator of the missing-data mechanism relies on the method of maximum likelihood, {already carried out by \citep{ibrahim1996parameter,ibrahim2001missing} outside the scope of deep learning.}
It is obtained by minimizing the negative observed log-likelihood \eqref{eq:loglikeli}: 
\begin{equation}\label{eq:estim_ll}
\hat{\theta}^L,\hat{\phi}^L=\mathrm{argmin}_{\theta \in \Theta,\phi \in \Phi} \quad \ell(\theta,\phi;x,y\odot r,r).
\end{equation}

We highlight the following points: 
\begin{itemize}[topsep=0pt,itemsep=1pt,leftmargin=*]
\item (Two-steps algorithm for SSL). Even if \eqref{eq:estim_ll} gives an estimator $\hat{\theta}^L$ of $\theta$, the latter can be really improved by incorporating the unlabeled data as in $\hat{{\mathcal{R}}}^{\mathrm{SSL}}_{\hat{\phi}}(\theta)$ in a second step (see Algorithm \ref{alg:first} in Section \ref{sec:alg}).
\item (MCAR setting). For not informative labels,  the unlabeled data are not used for the estimation, as noted in Section \ref{sec:notign}. Besides, as expected, the minimum of the function is attained for a mechanism equal to the proportion of the labeled data. Indeed, we have $\frac{\partial \ell(\theta,\phi_0)}{\partial \phi_0}=-\frac{n_l}{\phi_0}+ \frac{n-n_l}{1-\phi_0}$ and $\frac{\partial \ell(\theta,\phi_0)}{\partial \phi_0}=0 \Leftrightarrow \phi_0=\frac{n_\ell}{n}$.
\item (Convexity). The negative observed log-likelihood \eqref{eq:loglikeli} is convex in $\phi$, for a fixed $\theta \in \Theta$ (Appendix \ref{sec:app_estimationmecha}).
\end{itemize}
 \begin{remark}[Solving \eqref{eq:estim_ll} in practice]
To our knowledge, there is no closed form for the minimization problem. In practice, we propose to calculate the gradients by using the automatic differentiation package in PyTorch \citep{paszke2017automatic}. 
To comply with the constraint $\phi \in \Phi$, we consider $\sigma(\phi_k)=\frac{1}{1+\exp(-\phi_k)}, \forall k \in \mathcal{C}$ instead of $\phi_k$. In addition, we suggest solving \eqref{eq:estim_ll} subject to the constraint of $\sum_y\frac{\sum_{i=1}^n \mathbb{1}_{\{r=1,y_i=y\}}}{n}\frac{1}{\hat{\phi}_y^L}=1$, to comply with $\sum_y p(y)=1$ (see \eqref{eq:meca_mm_oracle}), by using the \href{https://github.com/crowsonkb/mdmm}{mdmm} package. 
\end{remark}

\subsection{Theoretical results}

In this section, we provide theoretical results that validate the relevance of the chosen estimators. The proofs are detailed in Appendix \ref{sec:proofs_Section4}. We first demonstrate the consistency of the moment estimator for a fixed $\theta \in \Theta$ by applying general results such as the law of large numbers and Slutsky's theorem.

\begin{proposition}[Consistency of $\hat{\phi}^M$]\label{prop:estim_m}
    The moment estimator defined by \eqref{eq:estim_moment} is consistent for a fixed $\theta \in \Theta$. 
\end{proposition}

Additionally, the consistency and asymptotic normality of the maximum likelihood estimator are obtained by applying Theorem 5.7 and Theorem 5.23 of \citep{van2000asymptotic} (stated in the more general case of M-estimators). We consider the negative observed log-likelihood for a fixed $\theta$, denoted as $\ell_\theta:\phi \mapsto \ell(\theta,\phi;x,y\odot r,r)$. In Appendix, we prove that the associated statistical model $\mathcal{P}_\phi=\{p(r|y\odot r;\phi):\phi \in \Phi \}$ is identifiable and that under interchangeability of differentiation w.r.t. $\phi$ and integration over $(x,y,r)$ (Assumption \ref{ass:ll1}), the Fisher information evaluated at the oracle estimate is invertible. Besides, we assume that $\phi$ is in the interior
of the set $\Phi=[0,1]^K$, i.e. it cannot be on its boundary (Assumption \ref{ass:ll2}). 

\begin{proposition}[Consistency, asymptotic normality of $\hat{\phi}^L$]\label{prop:estim_ll}
Under Assumptions \ref{ass:ll1} and \ref{ass:ll2}, the estimator $\hat{\phi}^L$ is consistent and asymptotically normal. 
\end{proposition}

Finally, the consistency of the estimator of the missing-data mechanism directly implies the consistency of the risks that we minimize in our SSL algorithms. 

\begin{proposition}[Consistency of the risk]\label{prop:risk}
If $\hat{\phi}$ is a consistent estimator of $\phi$ and if the mechanism is well specified, the risks $\hat{\mathcal{R}}_{\hat{\phi}}$ and $\hat{{\mathcal{R}}}^{\mathrm{SSL}}_{\hat{\phi}}(\theta)$ are consistent estimators of the theoretical risk $\mathcal{R}(\theta)=\mathbb{E}[\ell_\ell(\theta;x,y)]$.
\end{proposition}

\begin{remark}[Consistency of $\mathcal{\hat{R}}_{\hat{\phi}}$ {using unlabeled data}] As a consequence of the ignorability of the MCAR mechanism (see Section \ref{sec:notign}), the estimator of the theoretical risk using only labeled data is consistent in presence of MCAR labels. Proposition \ref{prop:risk} shows that the IPW estimator $\mathcal{\hat{R}}_{\hat{\phi}}$ is consistent for MNAR labels. It is worth noting that its expression refers only to labeled data but involves an estimator of the missing-data mechanism, computed on both labeled and unlabeled data (see Equations \eqref{eq:estim_moment} and \eqref{eq:estim_ll}). This underlines the relevance of unlabeled data in SSL with MNAR labels.
\end{remark}

\begin{remark}[Double-robustness of the SSL risk $\hat{{\mathcal{R}}}_{\hat{\phi}}^{\mathrm{SSL}}$]
An interesting property is double-robustness, meaning that the estimator is consistent even if either the missing-data mechanism estimation or imputation is inaccurate. \citet{hu2022on} prove that double-robustness of their debiased risk (see Section \ref{sec:compar}) holds, if they assume that under inaccurate propensity estimation, the imputation is perfect (in the sense that the model always predicts the right class). This is a strong assumption. In our work, double-robustness (in that sense) of $\mathcal{\hat{R}}_{\hat{\phi}}^{\textrm{SSL}}$ is directly implied by Proposition \ref{prop:risk} and by the unbiasedness of the risk. The strong assumption above could also be relaxed, applying Theorem 2 of \citep{miao2015identification} to our case: this is left as a perpective of our work.
\end{remark}

\subsection{Testing the assumption on the mechanism}\label{sec:test}
We present here a heuristics for estimating the missing-data mechanism to test if the labels are MCAR or not in the case of semi-supervised learning. The aim of such a test is to encourage the use of a specific method if the labels are not MCAR, or to support the selection of a traditional method if they are. 

We want to test
$$H_0: \: {\phi}\in \Phi^{\mathrm{MCAR}} \: \mathrm{against} \: H_1: {\phi}\notin \Phi^{\mathrm{MCAR}},$$
{where $\Phi^{\mathrm{MCAR}}=\{\phi\in [0,1]^K, \forall k,k', \phi_k=\phi_{k'}\}$.}
For the maximum likelihood estimator given in \eqref{eq:estim_ll}, we consider the following test statistic:
\begin{equation}\label{eq:stattest_ll}
-2\left(\inf_{\theta,\phi}\ell(\theta,{\phi})-\inf_{\theta} \ell(\theta,\phi^{\mathrm{MCAR}})\right).
\end{equation}
Under the same assumptions of Proposition \ref{prop:estim_ll}, we know that the test stastic $2(\ell_\theta(\phi)-\ell_\theta(\phi^{\mathrm{MCAR}})$, for a fixed $\theta \in \Theta$, converges
in distribution to a chi-squared random
variable $\chi^2_d$ (Theorem 16.7 of \citet{van2000asymptotic}), where $d$ is the difference in degrees of freedom between the null hypothesis ($H_0$) and the alternative hypothesis ($H_1$), i.e.\ $d=K-1$ for $K$ classes. We conjecture that an extension of Proposition \ref{prop:estim_ll} developed for a fixed $\theta \in \Theta$ can be obtained by considering the profile likelihood $\phi \mapsto \inf_{\theta \in \Theta} \ell(\theta,\phi)$, and by applying the asymptotic results on it \citep{murphy2000profile}. This implies that the test statistic \eqref{eq:stattest_ll} also converges in distribution to $\chi^2_d$. {Based on this asymptotic distribution, it is possible to calculate a p-value from \eqref{eq:stattest_ll} 
 and to easily test if the MCAR assumption is rejected.}

\subsection{Algorithms}\label{sec:alg}

To ensure clarity, we explain how the proposed estimators can be applied to any SSL algorithm. As previously mentioned, the goal is to plug-in the estimation of the mechanism into the estimators of the theoretical risk.

The moment estimator presented in \eqref{eq:estim_moment} is continuously updated throughout the SSL algorithm (Algorithm \ref{alg:second}), and thus the estimation of the mechanism does not add any additional computational cost when using the moment estimator, as the estimation of the mechanism and the model are performed in a single step. On the other hand, when using the maximum likelihood method, the estimation process is divided into two steps:
(i) the estimation of the mechanism by optimizing \eqref{eq:loglikeli} (Algorithm \ref{alg:first}) and (ii) the estimation of the model $\theta$ using the SSL algorithm (Algorithm \ref{alg:second} using the estimator (i) as input for $\hat{\phi}$). In both algorithms, the hyperparameters are classical:  the sizes of the mini-batch ($N_\mathcal{B}$,$N'_\mathcal{B}$), the learning rates ($\gamma_\phi$,$\gamma_\theta$,$\gamma'_\theta$) and the number of epochs ($N$, $N'$).

\begin{remark}[Adaptive threshold]
SSL algorithms \citep{rizve2021defense,sohn2020fixmatch} that utilize pseudo-label techniques 
often employ a fixed threshold to select relevant imputations from unlabeled data. However, several recent studies \citep{hu2022on,wei2021crest,berthelot2019remixmatch} have noted that an adaptive threshold can improve the performance of the classifier on the rarest classes, particularly in unbalanced semi-supervised learning. For instance, \citet{hu2022on} propose to use an adaptive threshold based on class proportions, setting higher requirements for popular classes and lower requirements for rare classes. Another possible adaptive threshold, suggested by our estimation of the missing-data mechanism, would depend on the missingness proportion of a class and set the highest requirement for the most observed class:
{\small
$$\forall k \in \mathcal{C}, \: \tau(y_k) = \tau_0 \left(\frac{\mathbb{P}(r=1|y=k)}{\max_y \mathbb{P}(r=1|y)}\right)^\beta = \tau_0 \left(\frac{\phi_y}{\max_y \phi_y}\right)^\beta,$$
}with $\tau_0$ the classical threshold and $\beta$ the hyper-parameter that determines how adaptive the cutoff is.
\end{remark}

\begin{algorithm}[tb]
	\caption{Maximum likelihood estimator for $\phi$}
	\label{alg:first}
	\begin{algorithmic}
	\small
		\STATE {\bfseries Input:} labeled data $D_\ell$, unlabeled data $D_u$
		\STATE Initialize $\theta_0$ (at random), $\phi_0$ (MCAR case: $n_\ell/n$). 
		\FOR{$k=0$ {\bfseries to} \textcolor{gray}{$N$} \textbf{iteratively}}     
		\STATE Sample a Mini-Batch $\mathcal{B}$ of size \textcolor{gray}{$N_{\mathcal{B}}$} from $D_\ell$ and from $D_u$.
		 \STATE $\phi_{k+1}=\phi_{k}-\textcolor{gray}{\gamma_\phi} 	\partial_\phi \frac{1}{N_{\mathcal{B}}}\sum_{i\in \mathcal{B}} \ell(\theta_k,\phi_k)$ \STATE Sample a Mini-Batch $\mathcal{B}$ of size \textcolor{gray}{$N_{\mathcal{B}}$} from $D_\ell$ and from $D_u$. 
        \STATE
        $\theta_{k+1}=\theta_{k}-\textcolor{gray}{\gamma_\theta}	\partial_\theta \frac{1}{N_{\mathcal{B}}}\sum_{i\in \mathcal{B}} \ell (\theta_k,\phi_k)$
		\ENDFOR
		\STATE {\bfseries Output:} ${\phi_{N}}=\hat{\phi}^L,{\theta_{N}}$
	\end{algorithmic}
    \end{algorithm}
    \begin{algorithm}[tb]
	\caption{Debiased SSL algorithm for informative labels}
	\label{alg:second}
	\begin{algorithmic}
	\small
		\STATE {\bfseries Input:} labeled data $D_\ell$, unlabeled data $D_u$, $\hat{\phi}$ (if available)
		\STATE Initialize $\theta_0$ (at random)
		\FOR{$k=1$ {\bfseries to} \textcolor{gray}{$N'$}}		    
		\STATE Sample a Mini-Batch $\mathcal{B}$ of size \textcolor{gray}{$N_{\mathcal{B}}'$} from $D_\ell$ and from $D_u$.  
            \IF{$\hat{\phi}$ is not provided}
            \STATE Compute $\hat{\phi}_y, \forall y \in \mathcal{C}$ by the method of moments \eqref{eq:estim_moment}. 
            \ENDIF
            \STATE $\theta_{k+1}=\theta_{k}-\textcolor{gray}{\gamma'_\theta}	\partial_\theta \frac{1}{N'_{\mathcal{B}}}\sum_{i\in \mathcal{B}} \hat{\mathcal{R}}_{\hat{\phi}}^{\mathrm{SSL}}(\theta_k)$
		\ENDFOR
		\STATE {\bfseries Output:} $\theta_{N'}$
	\end{algorithmic}
    \end{algorithm}

\section{Numerical experiments}\label{sec:xp}

In this study, we evaluate the effectiveness of our proposed estimates of the missing-data mechanism using the benchmark dataset MNIST \citep{lecun-mnisthandwrittendigit-2010}. Additionally, our debiased approach (Algorithm \ref{alg:second}) of the classical SSL method using pseudo-labels \citep{rizve2021defense} is compared in both its original implementation (Pl) and its debiased version for MCAR labels (\textbf{DePl}) \citep{schmutz2022don}, using both the MNIST dataset and two datasets of MedMNIST \citep{yang2021medmnist}.
Furthermore, we compare our debiased version of Fixmatch \citep{sohn2020fixmatch}, designed to handle informative labels, with its original counterpart (\textbf{Fix}) and its debiased version for MCAR labels (\textbf{DeFix}) \citep{schmutz2022don} on the CIFAR-10 dataset \citep{krizhevsky2009learning}.

To evaluate the accuracy of our proposed estimates of the missing-data mechanism, we calculate the normalized Mean Squared Error (MSE) using the formula $\|\hat{\phi}-\phi^\star\|_2/\|\phi^\star\|^2$. This provides a measure of how well our estimate of the missing-data mechanism ($\hat{\phi}$) approximates the true mechanism ($\phi^\star$). We consider four different estimators of the missing-data mechanism.
\begin{itemize}[topsep=0pt,itemsep=1pt,leftmargin=*]
\item \textbf{MLE}: the maximum likelihood estimator derived from Algorithm \ref{alg:first}. As highlighted in Section \ref{sec:maximum_ll}, we use the estimation of $\theta$ given by Algorithm \ref{alg:second} when assessing the model's performance.
\item \textbf{ME}: the moment estimator derived from Algorithm \ref{alg:second} by using a moving averaging strategy \eqref{eq:buffer} for the class distribution.
\item \textbf{MEg}: the moment estimator derived from Algorithm \ref{alg:second} by propagating the gradients through $\theta$.
\item \textbf{CADR}: the estimator derived from \citet{hu2022on}. Although the authors did not propose an estimation of the missing-data mechanism, we are able to derive it directly from their estimation of the class proportions (see \eqref{eq:estim_moment}).
\end{itemize}


\subsection{MNIST and CIFAR-10 for toy mechanisms}
\label{sec:bench}

The MNIST dataset is an advantageous choice for SSL as the classes are well-separated, allowing us to verify the effectiveness of our method in simple cases. 
In order to randomly select the labeled and unlabeled data per class according to a specific distribution, we follow the method proposed by \citet{hu2022on}. The number of labeled data (or unlabeled data) in each class $k$ is determined by $n_k = n_{1} \gamma^{-\frac{k-1}{K-1}}, \forall k\in \mathcal{C}$, where $\gamma$ controls the degree of imbalance among the classes, with $\gamma=1$ resulting in a balanced distribution of labeled data among classes. Additionally, $n_1$ represents the maximum (or minimum) number of labeled data among all the classes. In particular, we consider two cases (see Figure \ref{fig:acc_unbalanced_MNIST}, Appendix \ref{sec:add_xp}):
\begin{enumerate}[label=\textbf{S\arabic*.},topsep=0pt]
    \item \label{setting1} when the dataset is balanced, we randomly select labeled data in each class with $n_1=400$ and $\gamma=10$, and the remaining data is considered as unlabeled.
    \item  \label{setting2} when the dataset is unbalanced, we randomly select labeled data (resp. unlabeled data) with $n_1=400$ and $\gamma=10$ for  (resp. $\gamma=0.1$).
\end{enumerate}

\ref{setting1} (resp. \ref{setting2}) leads to a percentage of observed labels of $3\%$ (resp. 9\%). We trained a 3-layer CNN for both Algorithm \ref{alg:first} and \ref{alg:second}. 
In terms of estimation of the missing-data mechanism, all methods have comparable and low MSE values in the balanced setting \ref{setting1}, as reported in Appendix \ref{sec:add_xp} (Table \ref{table:MNIST balanced}). In the unbalanced case \ref{setting2}, the estimation of the missing-data mechanism \textbf{CADR} underestimates the observed proportions in the four rarest classes (i.e. classes 0 to 3), as seen in Figure \ref{fig:estim_meca}, which leads to a highest MSE (see Table \ref{table:MNIST_unbalanced}). For model estimation, while in the balanced case \ref{setting1} the methods have comparable results, there is an improvement in both test accuracy and test loss with our methods that include mechanism estimation (\textbf{MLE}, \textbf{MEg} and \textbf{ME}), especially for the less observed classes (i.e. classes 5 to 9). Note also that in both cases the method of \citet{hu2022on} has the highest test loss, which can be explained by the fact that the objective function that they minimize is quite different as explained in Section \ref{sec:compar}.

\begin{table}[h]
\caption{Test accuracy and test loss on MNIST, Setting \ref{setting2}.} 
\label{table:MNIST_unbalanced}
\begin{center}
\begin{scriptsize}
\begin{sc}
\begin{tabular}{cccc}
\toprule
\textbf{Method}   &\textbf{Loss} & \textbf{Accuracy} & \textbf{MSE $\phi$} \\
\midrule
Pl & $0.141 \pm 0.018 $ & $92.95 \pm 0.55$ & 0.594 \\
DePl     & $0.138 \pm 0.015$ & $93.18 \pm 0.71$ & 0.594 \\
\midrule
CADR & $0.160 \pm 0.029$ & $89.15 \pm 0.99$ & 
$0.106 \pm 0.012$ \\
MLE (Ours) & $0.116 \pm 0.021$ & $94.29 \pm 0.11$ & $0.027 \pm 0.012$ \\
MEg (Ours) &  $0.103 \pm 0.009$ & $94.83 \pm 0.38$ & $0.022 \pm 0.004$ \\
ME (Ours) & $0.111 \pm 0.005$ & $94.59 \pm 0.28$ & $0.037 \pm 0.002$\\
\bottomrule
\end{tabular}
\end{sc}
\end{scriptsize}
\end{center}
\end{table}

\begin{figure}
\centering
\includegraphics[width=0.23\textwidth]{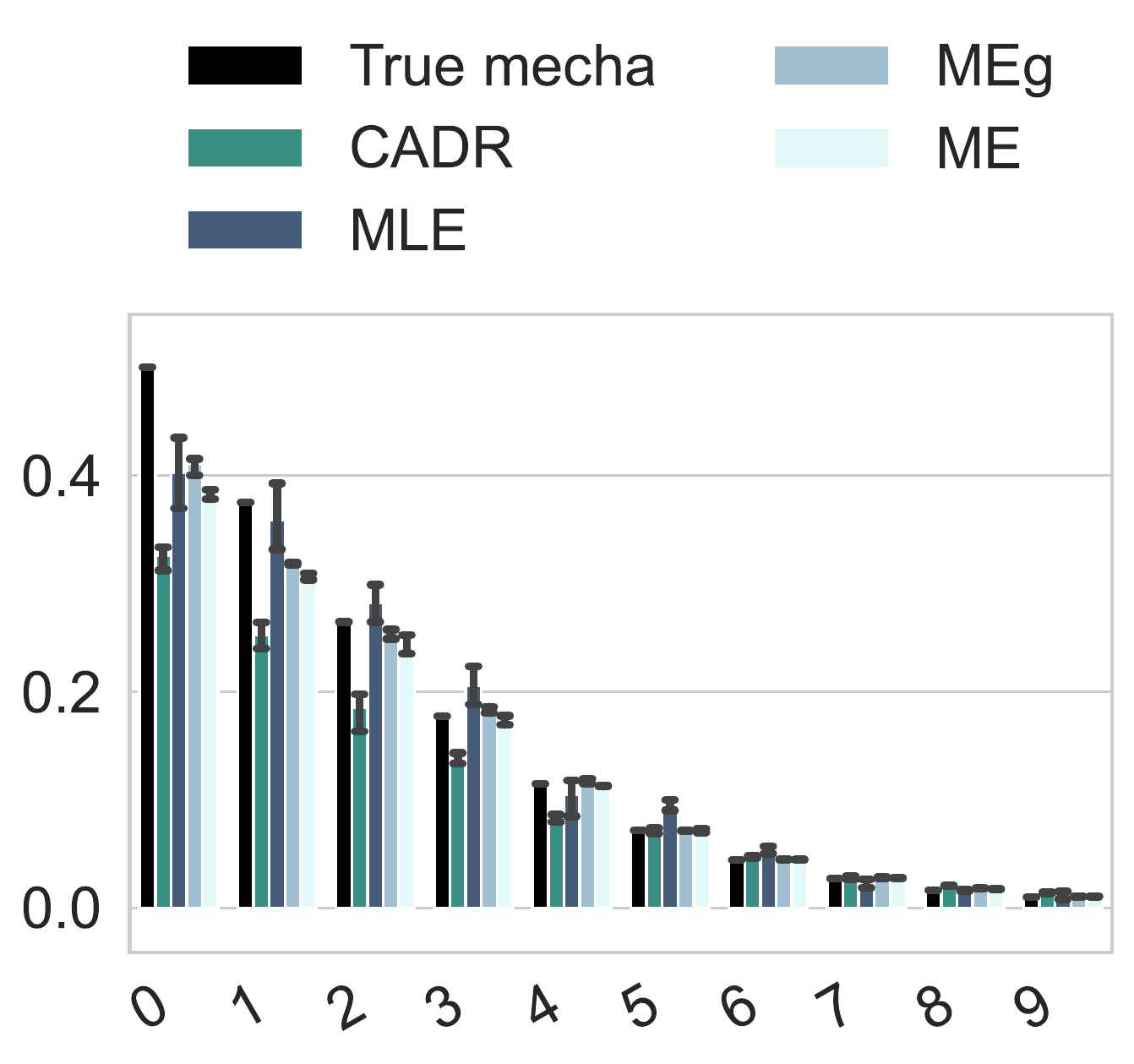}
\includegraphics[width=0.23\textwidth]{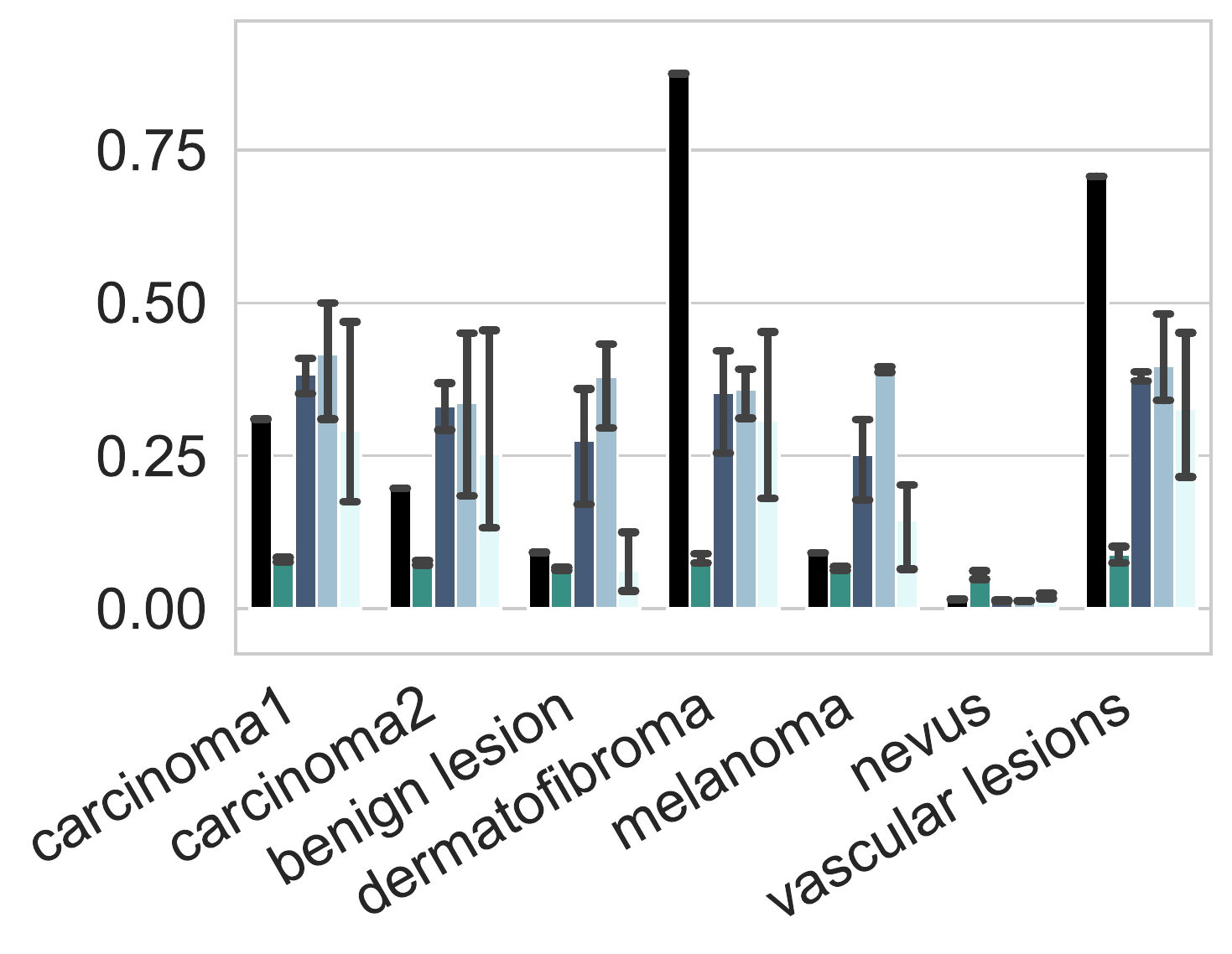}
 \caption{Estimation of the mechanism (coordinates of $\phi$ for each class) on MNIST (Setting \ref{setting2}) and on dermaMNIST.\label{fig:estim_meca}}
\end{figure}

In the CIFAR-10 dataset and considering Setting \ref{setting2}, we compare the original version of Fixmatch \citep{sohn2020fixmatch} 
with its debiased versions\footnote{Only the results for the moment estimator using averaging strategy (\textbf{ME}) are reported, as Algorithm \ref{alg:first} (\textbf{MLE}) has not yet been implemented with data augmentation and \textbf{MEg} proved difficult to calibrate. It can be challenging to find the right balance between too much initialization using $\hat{\phi}=n_\ell/n$ (leading to deviation from the optimal solution) and too little (leading to numerical problems). Therefore, we recommend considering the ME estimator when using data augmentation in practice.}. In Table \ref{table:CIFAR10_unbalanced}, our method \textbf{ME} demonstrates improved performance with higher overall accuracy. Again, we observe in Figure \ref{fig:acc_unbalanced_CIFAR10} (Appendix \ref{sec:add_xp}) that the classes with less observations (classes 5 to 10, particularly "dog" and "ship") are more accurately predicted when the missing-data mechanism is taken into account.

\begin{table}[h]
\caption{Test accuracy and test loss for CIFAR10, Setting \ref{setting2}.} 
\label{table:CIFAR10_unbalanced}
\vskip 0.15in
\begin{center}
\begin{scriptsize}
\begin{sc}
\begin{tabular}{ccc}
\toprule
\textbf{Method}   &\textbf{Loss} & \textbf{Accuracy}  \\
\midrule
Pl &  $0.426 \pm 0.017$ & $90.91 \pm 0.12$ \\
DePl      & $0.536 \pm 0.020$ & 
$89.71 \pm 0.16$\\
\midrule
CADR  & $0.452 \pm 0.006$& $91.14 \pm 0.30$ \\
ME (Ours) &  $0.321 \pm0.016$ & $91.88 \pm 0.24$\\
\bottomrule
\end{tabular}
\end{sc}
\end{scriptsize}
\end{center}
\vskip -0.1in
\end{table}

\subsection{medMNIST with pseudo-realistic mechanismns}\label{sec:medMNIST}

We first consider the dermaMNIST dataset \citep{codella2019skin}, which consists of 10,015 dermatoscopic images categorized into 7 different skin conditions. This dataset is unbalanced, its the most frequent class (71\% of the images) being benign naevi (aka moles). To simulate a more realistic MNAR scenario, we assume that a medical doctor would like to classify the conditions equally and select 70 images per class for labeling, resulting in 7\% of observed labels (see Figure \ref{fig:graphical}, left panel). Note that despite this selection, the dataset remains unbalanced due to the original distribution of the classes. Our three estimators (\textbf{MLE},\textbf{ME},\textbf{MEg}) of the missing-data mechanism detect that the class of naevi is very little observed compared to other classes (see Figure \ref{fig:estim_meca}), whereas \textbf{CADR} gives a mechanism where all classes are equally observed. 
In Table \ref{table:dermaMNIST}, only our methods determine if a lesion is a nevus or not with a high accuracy, which can be used as a pre-processing step before the images are reviewed by a medical doctor. Note that, even if the class proportions are known, \textbf{CADR} fails to give accurate results for the nevus class (see Table \ref{table:dermaMNIST_appendix} in Appendix \ref{sec:add_xp}), which shows the relevance of taking into account the missing-data mechanism in this case. Finally, we use \textbf{MLE} together with the test presented in Section \ref{sec:test} to assess whether or not labels are informative. If we generate MCAR labels, the likelihood ratio test is rigthfully unable to reject the MCAR hypothesis ($p$-value of 0.68$\pm$0.2 over 10 runs). But if we give it images with MNAR labels, the test rejects the MCAR hypothesis with very high confidence ($p$-value$<10^{-4}$ for all 10 runs).


\begin{table}[h]
\caption{Test accuracies on dermaMNIST and noduleMNIST3D. 
}
\label{table:dermaMNIST}
\vskip 0.15in
\begin{center}
\begin{scriptsize}
\begin{sc}
\begin{tabular}{cc|c}
\toprule
\textbf{Method}   & \textbf{dermaMNIST} & \textbf{noduleMNIST}  \\
\midrule
Pl & $57.72 \pm 1.95$ & 84.91\\
\midrule
CADR & $49.36 \pm 1.91$ & 80.32\\
MLE (Ours) &   $66.4 \pm 0.81$& 85.8 \\
MEg (Ours) & $66.65 \pm 1.76$  & 82.26\\
ME (Ours) & $65.8 \pm 0.78$  & 85.16 \\
\bottomrule
\end{tabular}
\end{sc}
\end{scriptsize}
\end{center}
\vskip -0.1in
\end{table}

We now consider the noduleMNIST3D dataset \citep{armato2011lung} on images from thoracic CT scans, which is of particular interest to simulate the MNAR labels. We have access to the subtlety score $s$, which describes from 1 (extremely subtle) to 5 (obvious) the difficulty of nodule detection. According to these scores, we simulate the missing-data mechanism using $p(r|y)=\sum_{s=1}^5 p(r|s)p(s|y)$, with $p(s|y)$ computed on the data. The only quantities to choose are the probability of being observed given the subtlety score, we fix a low probability when the detection was difficult ($p(r|s\in \{1,2,3\})=0.1$) and a high probability when the detection was easy ($p(r|s\in \{4,5\})=0.9$). At the end, the class of benign nodules has a missing-data proportion (43\%) higher than the class of malignant nodules (8\%). On the contrary to the missing-data setting chosen for dermaMNIST, the more observed class is also the less frequent (Figure \ref{fig:noduleMNIST}). \textbf{MLE} performs better in terms of accuracy (Table \ref{table:dermaMNIST}) and all the methods designed for informative labels has a highest specificity than the classical one \textbf{Pl} for the class of malignant nodules (Table \ref{tab:noduleMNIST} in Appendix \ref{sec:add_xp}).

\section{Conclusion}

For future works, we would be eager to (i) provide a more realistic theoretical grounding without freezing $\theta$ and (ii) propose another statistical test for the moment estimator, for example using bootstrap strategies. Note that the latter perspective is quite challenging, because of the sample bias in the informative data.




\bibliography{biblio_ICML}
\bibliographystyle{icml2022}


\newpage
\appendix
\onecolumn

\section{Proof of identification}\label{sec:proof_id}

\begin{proposition}[Identification of the joint distribution]
Under Assumptions \ref{ass:ourSSL2} (self-masked MNAR), the joint distribution $p(y,x,r)$ is identified. 
\end{proposition}
\begin{proof}
This proof is a direct application of Theorem 1 of \citep{miao2015identification} (stated in a more general case). For clarity, we give it again here in our setting for clarity by proving the intermediate results. 

Let us consider the odds ratio $OR(y)=\frac{p(y|r=0)p(y=1|r=1)}{p(y|r=1)p(y=1|r=0)}$ where $y=1$ is used as a reference value. 
The goal is to determine $p(y|r=0,x)$ by only quantities involving observed data, and it will imply that the joint distribution $p(y,x,r)$ is identified. 
Proposition 2 of \citep{miao2015identification} gives the two following equalities:
\begin{align}
\label{eq:id_2}
p(y|r=0,x)&=\frac{OR(y)p(y|r=1,x)}{\mathbb{E}[OR(y)|r=1,x]} \\
\label{eq:id_1}
\mathbb{E}[\tilde{OR}(y)|r=1,x]&=\frac{p(x|r=0)}{p(x|r=1)} \mbox{ with $\tilde{OR}(y)=\frac{OR(y)}{\mathbb{E}[OR(y)|r=1]}$}
\end{align}
The first equation \eqref{eq:id_2} indicates that the identification of the odds ratio function involves the identification of $p(y|r=0,x)$. Note that in \eqref{eq:id_1}, as $p(x|r=0), p(x|r=1)$ and $p(y|r=1,x)$ are obtained from observed data, so is $\mathbb{E}[\tilde{OR}(y)|r=1,x]$. We just have to prove that \eqref{eq:id_2} has a unique solution. Let us consider $\tilde{OR}^\star(y)\neq \tilde{OR}^\star(y)$, we have
$$\mathbb{E}[\tilde{OR}^\star(y)|r=1,x]=\frac{p(x|r=0)}{p(x|r=1)},$$ which implies $$\mathbb{E}[\tilde{OR}^\star(y)-\tilde{OR}(y)|r=1,x]=0 \Leftrightarrow \tilde{OR}^\star(y)=\tilde{OR}(y),$$
This is obtained by using Condition 1 of \citep{miao2015identification}. In our case, it amounts to assuming that for any image, there is a non-zero probability that it has any label.
$\tilde{OR}(y)$ is identified and so is ${OR}(y)$, noting that $OR(y)=\frac{\tilde{OR}(y)}{\tilde{OR}(y=0)}$.

\end{proof}

\section{Proofs of the theoretical results of Section \ref{sec:estim_meca}}
\label{sec:proofs_Section4}

\subsection{Moment estimator}

The moment estimator has the following form, for a fixed $\theta \in \Theta$:
$$
\hat{\phi}_y^{M}=\frac{\sum_{i=1}^n \mathbb{1}_{\{r=1,y_i=y\}}}{n}\frac{1}{\hat{p}(y,\theta)}, \forall y \in \mathcal{C}$$
with $$\hat{p}(y,\theta):=\hat{p}_\theta(y)=\frac{1}{n}\sum_{i=1}^{n}p(y_i|x_i;\theta)$$
In this section, we prove the consistency of this estimator. 

\begin{proposition}[Consistency of $\hat{\phi}^M$]
    The moment estimator defined by \eqref{eq:estim_moment} is consistent for a fixed $\theta \in \Theta$. 
\end{proposition}
\begin{proof}
We have by the law of large numbers, 
\begin{align}
\label{eq:proba_1}
\frac{1}{n}\sum_{i=1}^n\mathbb{1}_{\{r_i=1,y_i=y\}} &\underset{n \to +\infty}{\overset{\mathbb{P}}{\longrightarrow}} \mathbb{E}[\mathbb{1}_{\{r_i=1,y_i=y\}}]=p(r=1,y) \\
\label{eq:proba_2}
\frac{1}{n}\sum_{i=1}^n p(y_i|x_i;\theta)&\underset{n \to +\infty}{\overset{\mathbb{P}}{\longrightarrow}} \mathbb{E}[p(y|x)]=p(y)
\end{align}
We can apply the continuous mapping theorem to $f(x)=1/x, x\in]0,1]$ (assuming that the probability of each class is greater than 0) to get
\begin{equation}\label{eq:proba_3}
\frac{1}{\frac{1}{n}\sum_{i=1}^n p(y_i|x_i)} \underset{n \to +\infty}{\overset{\mathbb{P}}{\longrightarrow}} \frac{1}{p(y)}
\end{equation}

In \eqref{eq:proba_1} and \eqref{eq:proba_3}, we have the convergence in probability to a constant, which implies by Slutsky's theorem that the product converges in probability to the product of the constants (Slutsky's theorem gives the property in distribution but convergence in probability and law is equivalent if the limit is a constant), which gives:

$$\forall y \in \mathcal{C}, \hat{\phi}_y^M \underset{n \to +\infty}{\overset{\mathbb{P}}{\longrightarrow}} \phi_y=\frac{p(r=1,y)}{p(y)}$$

\end{proof}

\subsection{Maximum likelihood estimator}\label{sec:app_estimationmecha}

In this section, we study the negative observed log-likelihood $\phi\mapsto \ell(\theta,\phi;x,y\odot r,r)$ and derive the consistency and the asymptotic normality of the maximum likelihood estimator $\phi^L$ given in \eqref{eq:estim_ll}. All the results are obtained for a fixed $\theta \in \Theta$. 
Let us recall: 
\begin{equation}\label{eq:loglikeli_app}
\ell(\theta,\phi;x,y\odot r,r)= -\sum_{i=1}^{n} r_i \log p(y_i|x_i;\theta)\phi_{y_i}
- \sum_{i=1}^{n} (1-r_i) \log \sum_{\tilde{y} \in \mathcal{C}} p(\tilde{y}|x_i;\theta)(1-\phi_{\tilde{y}})
\end{equation}
For simplicity, as $\theta$ is fixed in this section, let us define the function $\ell_\theta:\mathbb{R^K}\mapsto\mathbb{R}$  such that $\ell_\theta(\phi)=\ell(\theta,\phi,x,y\odot r,r)$.

\subsubsection{Convexity}

\begin{proposition}[Convexity of $\ell_\theta(\phi;x,y\odot r,r)$]\label{prop:conv}
For a fixed $\theta \in \Theta$, the negative observed log-likelihood $\phi\mapsto \ell(\theta,\phi;x,y\odot r,r)$ is convex.
\end{proposition}

\begin{proof}
Let us remark that:
\begin{align*}
\ell_\theta(\phi)&=-\sum_{i=1}^{n} r_i \log p(y_i|x_i)\phi_{y_i}
- \sum_{i=1}^{n} (1-r_i) \log \sum_{\tilde{y} \in \mathcal{C}} p(\tilde{y}|x_i)(1-\phi_{\tilde{y}}) \\
&=-\sum_{i=1}^{n} r_i \log p(y_i|x_i)\phi_{y_i}
- \sum_{i=1}^{n} (1-r_i) \log (-\phi_k B_{i,k} + C_{i,k}),
\end{align*}
where $B_{i,k}=p(y_i=k|x_i;\theta)$ and $C_{i,k}=B_{i,k}+\sum_{\tilde{y}\in \mathcal{C}\setminus\{k\}} p(\tilde{y}|x_i;\theta)(1-\phi_{\tilde{y}}).$
Let us compute the Hessian $H$ of $F$, such that
$$H(\phi):=\left(\frac{\partial^2 \ell_\theta}{\partial \phi_\ell\partial \phi_k} (\phi)\right)_{k,\ell \in \mathcal{C}}$$
We have:
$$\forall k \in \mathcal{C}, \frac{\partial \ell_\theta}{\partial \phi_k}(\phi)=-\left(\sum_{i=1}^{n} r_i \frac{1}{\phi_k}\mathbb{1}_{y_i=k} - \sum_{i=1}^n (1-r_i)\frac{B_{i,k}}{-\phi_k B_{i,k} + C_{i,k}} \right).$$
Thus, 
\begin{align*}
\forall k \in \mathcal{C}, \frac{\partial^2 \ell_\theta}{\partial \phi_k\partial \phi_k}(\phi)&=\sum_{i=1}^{n} r_i \frac{1}{\phi_k^2}\mathbb{1}_{y_i=k} + \sum_{i=1}^n  (1-r_i)\frac{(B_{i,k})^2}{(-\phi_k B_{i,k} + C_{i,k})^2} \\
\forall \ell \in \mathcal{C}, \ell\neq k, \frac{\partial^2 \ell_\theta}{\partial \phi_\ell\partial \phi_k}(\phi)&= \sum_{i=1}^n  (1-r_i)\frac{B_{i,k}B_{i,\ell}}{(-\phi_k B_{i,k} + C_{i,k})^2}
\end{align*}
In the following, let us denote $A_i=-\phi_kB_{i,k}+C_{i,k}=\sum_{\tilde{y}\in \mathcal{C}} p(\tilde{y}|x_i;\theta)(1-\phi_{\tilde{y}})$
We have now to prove that the Hessian is positive-definite . As it is a symmetric matrix, we can show that $\forall v \in \mathbb{R}^k, v\neq\vec{0}, v^THv > 0$.

$$v^THv=\sum_{k=1}^K \frac{v_k^2}{\phi_k^2}\sum_{i=1}^nr_i\mathbb{1}_{y_i=k} + \underbrace{\sum_{k=1}^K v_k^2 \sum_{i=1}^n (1-r_i) \frac{(B_{i,k})^2}{A_i^2}+2\sum_{1\leq k<\ell \leq K} v_kv_\ell \sum_{i=1}^n (1-r_i) \frac{B_{i,k}B_{i,\ell}}{A_i^2}}_{\textrm{=T}} > 0$$

The first term is trivially greater or equal to 0. Moreover, it is never equal to $0$, if at least one sample is observed ($n_\ell>0$). For the last two terms, note that: 
$$T=\sum_{i=1}^n\left(\sum_{k=1}^K v_kB_{i,k}\sqrt{\frac{(1-r_i)}{A_i^2}}\right)^2\geq 0$$

\end{proof}

\begin{remark}[Domain of definition of $\ell_\theta(\phi;x,y\odot r,r)$]\label{rem:min}
We look at the natural domain of the negative observed log-likelihood $\phi\mapsto \ell(\theta,\phi;x,y\odot r,r)$, for a fixed $\theta \in \Theta$. The goal is to know if we can minimize the function without constraint on $\phi$ (if its domain of definition is included in $[0,1]$) 

In \eqref{eq:loglikeli_app}, the first term implies $\forall y_i, p(y_i|x_i;\theta)\phi_{y_i}>0$ i.e. $\phi_{k}>0, \:  \forall k\in \{1,\dots,K\}$. The second term requires $\forall i \in \{n_\ell+1;n\}$,
\begin{align*}
&\sum_{\tilde{y} \in \mathcal{C}} p(\tilde{y}|x_i;\theta)(1-\phi_{\tilde{y}})>0 \\
\Leftrightarrow& \forall k \in \{1,\dots,K\}, (1-\phi_k)p(\tilde{y}=k|x_i;\theta)+\sum_{\tilde{y}\in \mathcal{C}\setminus\{k\}} p(\tilde{y}|x_i;\theta)(1-\phi_{\tilde{y}}) > 0 \\
\Leftrightarrow& \forall k \in \{1,\dots,K\}, \phi_k< 1 + \frac{1}{p(\tilde{y}=k|x_i;\theta)}\sum_{\tilde{y}\in \mathcal{C}\setminus\{k\}} p(\tilde{y}|x_i;\theta)(1-\phi_{\tilde{y}})
\end{align*}
As $\frac{1}{p(\tilde{y}=k|x_i;\theta)}\sum_{\tilde{y}\in \mathcal{C}\setminus\{k\}} p(\tilde{y}|x_i;\theta)(1-\phi_{\tilde{y}})>0$, this last inequality does not necessarily implies $\phi_k\leq1$. Therefore, a reparametrization trick or constrained optimization is essential. The domain of definition for $\phi_k, k \in \{1,\dots,K\}$ of the negative log-likelihood $\ell$ is then $$D_{{\phi_k}}=\left]0;1+\min_{i\in \{1,\dots,n\}} \frac{1}{p(\tilde{y}=k|x_i;\theta)}\sum_{\tilde{y}\in \mathcal{C}\setminus\{k\}} p(\tilde{y}|x_i;\theta)(1-\phi_{\tilde{y}})\right[.$$ 
\end{remark}

\subsubsection{Consistency and asymptotic normality}

The consistency and asymptotic normality of the maximum likelihood estimator is obtained by applying Theorem 5.7 and Theorem 5.23 of \citep{van2000asymptotic}.  
Let us assume the following:
\begin{enumerate}[label=\textbf{A\arabic*.}]
    \setcounter{enumi}{2}
    \item  \label{ass:ll1} We can interchange differentiation with respect to $\phi$ and integration over $(x,y,r)$. 
    \item \label{ass:ll2} $\forall k \in \mathcal{C}$, there exists a compact interval $U_k$ such that $\phi_k\in U_k\subset \: ]0,1[$. 
\end{enumerate}

To get the results, we need first to show the identifiability of the statistical model $\mathcal{P}_\phi=\{p(r|y\odot r;\phi):\phi \in \Phi \}$ and the nonsingularity of the Fisher information $I_{\phi^\star}$, with $\phi^\star$ the oracle point (i.e. $\phi^\star=\mathrm{argmin}_{\phi \in \Phi} \ell_\theta(\phi)$). 

\begin{lemma}[Identifiability of $\mathcal{P}_\phi$]\label{prop:id}  The model $\mathcal{P}_\phi$ is identifiable. 
\end{lemma}
\begin{proof}
Let us consider $\phi,\phi'\in \Phi$ such that $p(r|y\odot r;\phi) =p(r|y\odot r;\phi') $. 
\begin{align*}
&p(r|y\odot r;\phi) =p(r|y\odot r;\phi'), \forall y,r \\
&\Leftrightarrow 
r\phi_y + \sum_{\tilde{y}}(1-r)(1-\phi_{\tilde{y}}) = r\phi'_y + \sum_{\tilde{y}}(1-r)(1-\phi'_{\tilde{y}}), \: \forall y,r \\
&\Leftrightarrow 
r(\phi_y-\phi'_y) +  \sum_{\tilde{y}}(1-r)(\phi'_y-\phi_{\tilde{y}}) = 0, \:  \forall y,r
\end{align*}
The case $r=1$ leads to $\phi_y=\phi'_y, \forall y$. 
\end{proof}

\begin{lemma}[Nonsingularity of the Fisher information $I_{\phi^\star}$]\label{prop:Fisher} Under Assumption \ref{ass:ll1}, the Fisher information at the oracle point $\phi^\star$ is nonsingular. 
\end{lemma}
\begin{proof}
Assumption \ref{ass:ll1} implies that $\forall \ell,k \in \mathcal{C}, (I_{\phi^\star})_{(\ell,k)}=-\mathbb{E}_{x,y,r}\left[\frac{\partial^2\log\ell_\theta}{\partial\phi_\ell\partial\phi_k}(\phi^\star)\right]=-\frac{1}{n}\mathbb{E}_{x,y,r}[(H(\phi))_{(\ell,k)}]$. We can simply use the strict convexity of the function $\phi \mapsto \ell_\theta(\phi)$ proved in Proposition \ref{prop:conv}.
\end{proof}

\begin{proposition}[Consistency of $\hat{\phi}^L$]\label{prop:consistency_ll}
    Under Assumption \ref{ass:ll2}, the maximum likelihood estimator $\hat{\phi}^L$ defined in \eqref{eq:estim_ll} is consistent, for a fixed $\theta \in \Theta$. 
\end{proposition}

\begin{proof}
As said in Section 5.5 of \citet{van2000asymptotic} for the application to the maximum likelihood estimators, the method consists of applying Theorem 5.7 of \citet{van2000asymptotic} by noting that the MLE is a M-estimator: $\hat{\phi}^L \in \mathrm{argmin}_{\phi\in \Phi} M_n(\phi)$, with $M_n(\phi)=\frac{1}{n}\sum_{i=1}^n\log\frac{p(x_i,y_i\odot r_i,r_i;\phi^\star)}{p(x_i,y_i\odot r_i,r_i;\phi)}$ and $M(\phi)=\mathbb{E}\left[\log\frac{p(x,y\odot r,r;\phi^\star)}{p(x,y\odot r,r;\phi)}\right]$.
\begin{itemize}
\item We have the identifiability of $\mathcal{P}_\phi$ by Lemma \ref{prop:id}. We also have the strong identifiability, which is equivalent to the identifiability, since the restricted set of $\Phi$, $\otimes_{k\in \mathcal{C}} U_k$, is compact (see Assumption \ref{ass:ll2}).
\item We have to show that the uniform weak law of large numbers hold for the function $\phi\mapsto \log\frac{p(x_i,y_i\odot r_i,r_i;\phi^\star)}{p(x_i,y_i\odot r_i,r_i;\phi)}$, i.e.\ $$\mathrm{sup}_{\phi \in \Phi} |M_n(\phi)-M(\phi)|\underset{n \to +\infty}{\overset{\mathbb{P}}{\longrightarrow}} 0.$$
Following Theorem 4.2 of \citet{wainwright2019high}, a sufficient condition is that there exists $M>0$ sucht that, $\left|\log\frac{p(x,y\odot r,r;\phi^\star)}{p(x,y\odot r,r;\phi)}\right|<M, \: \forall \phi,x,y,r,$ i.e. the uniform boundedness. We will prove that there exists $M_1,M_2>0$ such that
$M_1\leq \frac{p(x,y\odot r,r;\phi^\star)}{p(x,y\odot r,r;\phi)} \leq M_2$.

We have
$$\frac{p(x,y\odot r,r;\phi^\star)}{p(x,y\odot r,r;\phi)}=\frac{r\phi^\star_y+(1-r)\sum_{\tilde{y}}(1-\phi^\star_{\tilde{y}})}{r\phi_y+(1-r)\sum_{\tilde{y}}(1-\phi_{\tilde{y}})}$$

Using Assumption \ref{ass:ll2}, we get $M_1=\min\left(\frac{\min_k a_k}{\max_k b_k},\frac{1-\max_k b_k}{1-\min_k a_k}\right)$ and $M_2=\max\left(\frac{1-\min_k a_k}{1-\max_k b_k},\frac{\max_k b_k}{\min_k a_k}\right)$

\item By identifiability of the statistical model $\mathcal{P}_\phi$, we have: $M(\phi)=M(\phi^\star)$ implies $\phi=\phi^\star$. Therefore $\forall \phi \neq \phi^\star, M(\phi)>M(\phi^\star)$ and the function $\phi \mapsto M(\phi)$ admits a strict minimum in $\phi_0$.
\end{itemize}
\end{proof}

\begin{proposition}[Asymptotic normality of $\hat{\phi}^L$]
    Under Assumption \ref{ass:ll1} and \ref{ass:ll2}, the maximum likelihood estimator $\hat{\phi}^L$ defined in \eqref{eq:estim_ll} is asymptotically normal, for a fixed $\theta \in \Theta$. 
\end{proposition}

\begin{proof}
The proof directly follows from Theorem 5.39 of \citet{van2000asymptotic} (Corollary of Theorem 5.23 for the maximum likelihood estimators). To apply the theorem, we check the following conditions:
\begin{itemize}
\item The statistical model $\mathcal{P}_\phi=\{p(r|y\odot r;\phi):\phi \in \Phi \}$ is differentiable in quadratic mean at $\phi^\star$, because $p(r|y\odot r;\phi)=r\phi_y + \sum_{\tilde{y}}(1-r)(1-\phi_{\tilde{y}})$ is trivially twice differentiable. 
\item 
The score function $S(\phi;r,y\odot r)$ is uniformly bounded in $y,r$ and $\phi$, ranging over a compact and continuous in $\phi$ forall $y,r$, i.e. we want to show that there exists a real number $M$, 
$$\|S(\phi;r,y)\|_1=\sum_{k=1}^K \left|\frac{\partial \log p(r|y\odot r;\phi)}{\partial \phi_k}\right|\leq M, \forall \phi,y,r$$

The score function for its coordinate $k$ is: 
$S(\phi_k;r,y\odot r)=\frac{\partial \log p(r|y\odot r;\phi)}{\partial \phi_k}=\frac{r\mathbb{1}_{y=k}-(1-r)}{r\phi_y+(1-r)\sum_{\tilde{y}}(1-\phi_{\tilde{y}})}$. If $r=1$, this amounts to bound $1/\phi_k$ and if $r=0$, this amounts to bound $\frac{1}{\sum_{\tilde{y}}(1-\phi_{\tilde{y}})}\leq\frac{1}{K(1-\max_k \phi_k)}$. Using Assumption \ref{ass:ll2} is sufficient to get the bound. Let us denote $U_k=[a_k,b_k], \forall k \in \mathcal{C}$, one has $a_k\leq \phi_k \leq b_k \forall k \in \mathcal{C}$ and we can choose for the bound $M=K\max\left(\frac{1}{\min_k a_k},\frac{1}{K(1-\max_k b_k)}\right)$. 
\item By Lemma \ref{prop:Fisher}, the Fisher information $I_{\phi^\star}$ is nonsingular.
\item By Proposition \ref{prop:consistency_ll}, the estimator $\hat{\phi}^L$ is consistent.
\end{itemize}

\end{proof}

\subsection{Consistency of the SSL risk}

\begin{proposition}[Consistency of the risk]
If $\hat{\phi}$ is a consistent estimator of $\phi$ and if the mechanism is well specified, the risks $\hat{\mathcal{R}}_{\hat{\phi}}$ and $\hat{{\mathcal{R}}}^{\mathrm{SSL}}_{\hat{\phi}}(\theta)$ are consistent estimators of the theoretical risk $\mathcal{R}(\theta)=\mathbb{E}[\ell_\ell(\theta;x,y)]$.
\end{proposition}

\begin{proof}
We prove the results for the IPW estimator $\hat{\mathcal{R}}_{\hat{\phi}}$ (the proof is similar for $\hat{{\mathcal{R}}}^{\mathrm{SSL}}_{\hat{\phi}}(\theta)$). 

It is a simple application of the law of large numbers, using the unbiasedness of the estimator (by Proposition \ref{prop:IPWunbiased}).

We have:
\begin{align*}
\frac{1}{n}\sum_{i=1}^n r_i \frac{\ell_\ell(\theta;x_i,y_i)}{\hat{\phi}_{y_i}}=\frac{1}{n}\sum_{i=1}^n r_i \frac{\ell_\ell(\theta;x_i,y_i)}{{\phi}_{y_i}} + o_\mathbb{P}(1) \underset{n \to +\infty}{\overset{\mathbb{P}}{\longrightarrow}} \mathbb{E}\left[r\frac{\ell_\ell(\theta;x,y)}{\phi_y}\right]=\mathbb{E}\left[{\ell_\ell(\theta;x,y)}\right]
\end{align*}
\end{proof}

\section{Additional numerical experiments}
\label{sec:add_xp}

\begin{figure}[H]    
   \centering \includegraphics[width=0.35\textwidth]{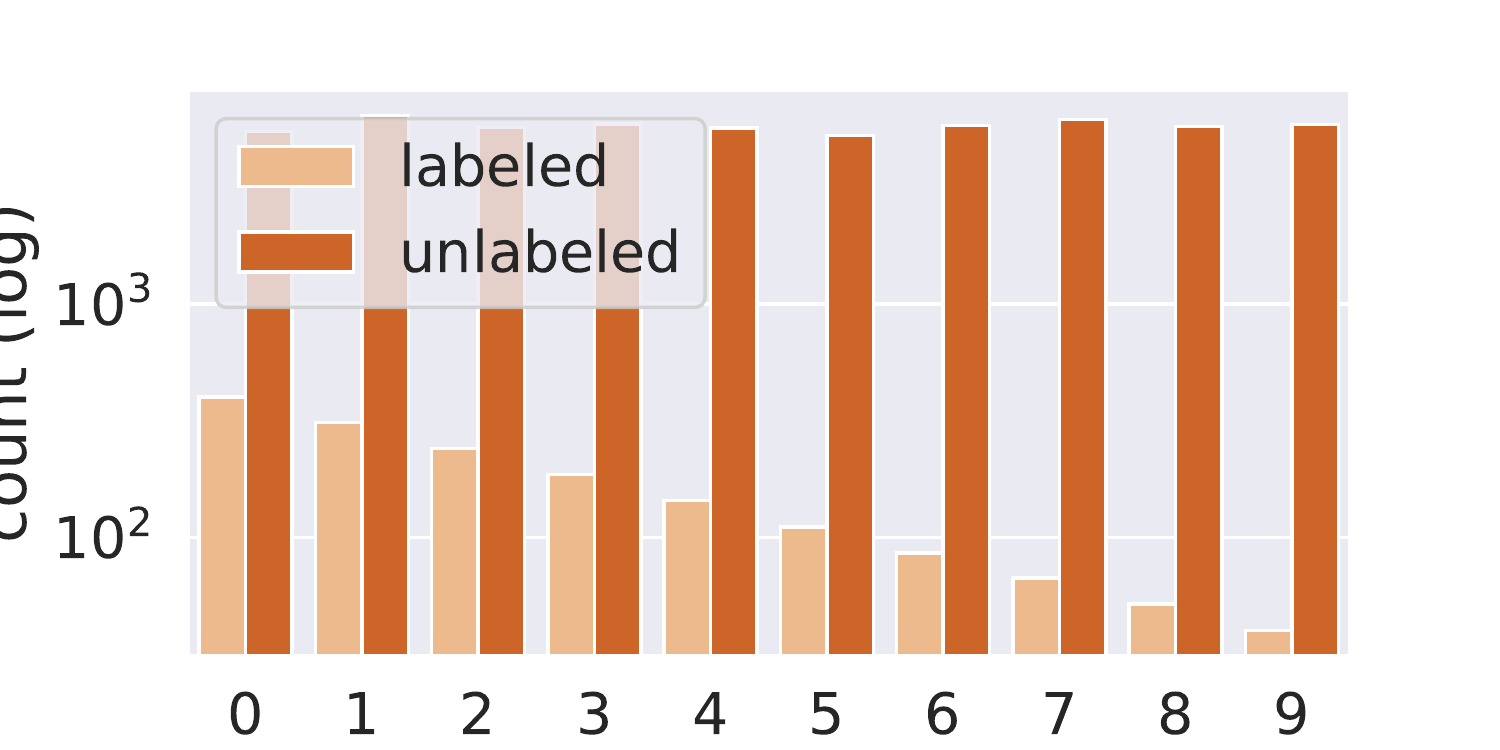}
         \includegraphics[width=0.35\textwidth]{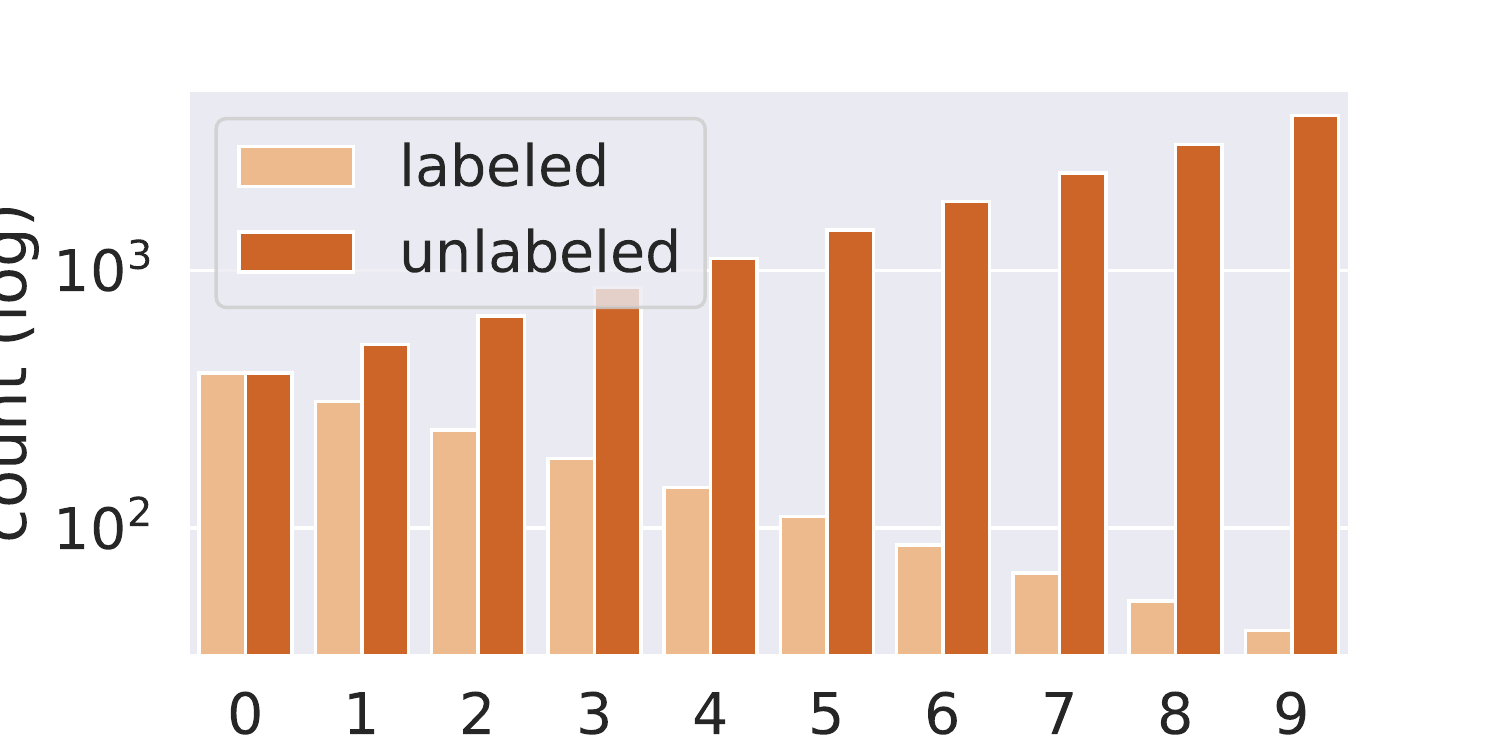}
         \caption{MNAR labels on MNIST (Settings \ref{setting1} and \ref{setting2})}
         \label{fig:unbalanced_MNIST}
\end{figure}

\begin{table}[H]
\caption{Test accuracy and test loss on MNIST, Setting \ref{setting1}. \label{table:MNIST balanced}}
\vskip 0.15in
\begin{center}
\begin{small}
\begin{sc}
\begin{tabular}{cccc}
\toprule
  \textbf{Method}   &\textbf{Loss} & \textbf{Accuracy} & \textbf{MSE $\phi$} \\
\midrule
Pl & $0.259  \pm 0.034$ & $95.48 \pm 0.16$ & 0.318 \\
DePl      & $0.237 \pm 0.045$ & $95.69 \pm 0.06$ & 0.318\\
\bottomrule
CADR & $0.272 \pm 0.046$ & $95.40 \pm 0.33$ & $0.014 \pm 0.004$ \\
MLE (Ours) & $0.249  \pm 0.050$ & $95.59  \pm 0.40$ & $0.031  \pm 0.009$\\
MEg (Ours) &  $0.240 \pm 0.029$ & $95.69 \pm 0.80$ & $0.004 \pm 0.001$\\
ME (Ours) & $0.240 \pm 0.027$& $95.34 \pm 0.26$& $0.013 \pm 0.001$\\
\midrule
\end{tabular}
\end{sc}
\end{small}
\end{center}
\vskip -0.1in
\end{table}

 \begin{figure}
 \centering
\includegraphics[width=0.6\textwidth]{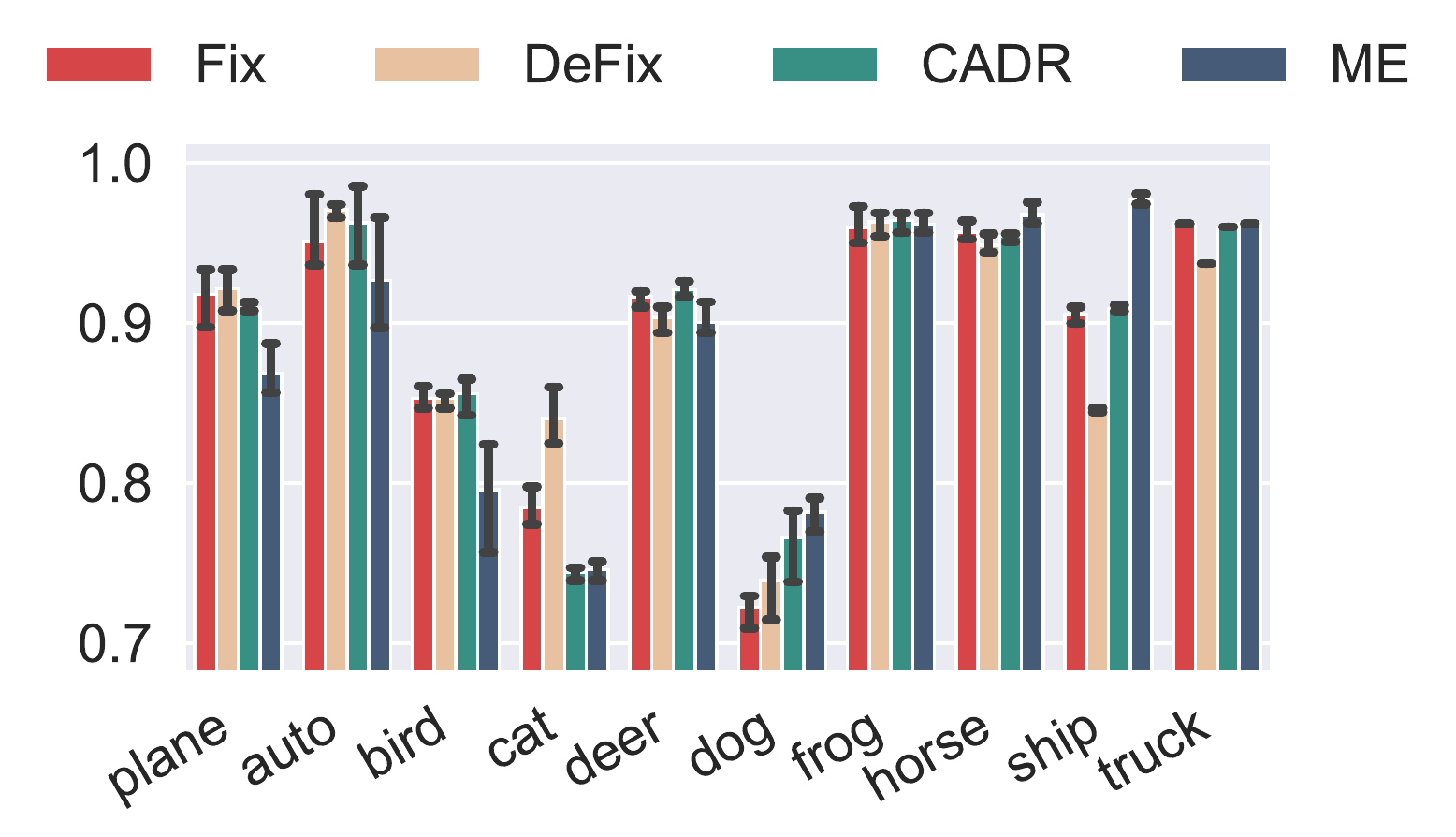}
 \caption{Accuracy per class on CIFAR10, Setting \ref{setting2}.}
 \label{fig:acc_unbalanced_CIFAR10}
\end{figure}

 \begin{figure}[H]
 \centering
 \includegraphics[width=0.6\textwidth]{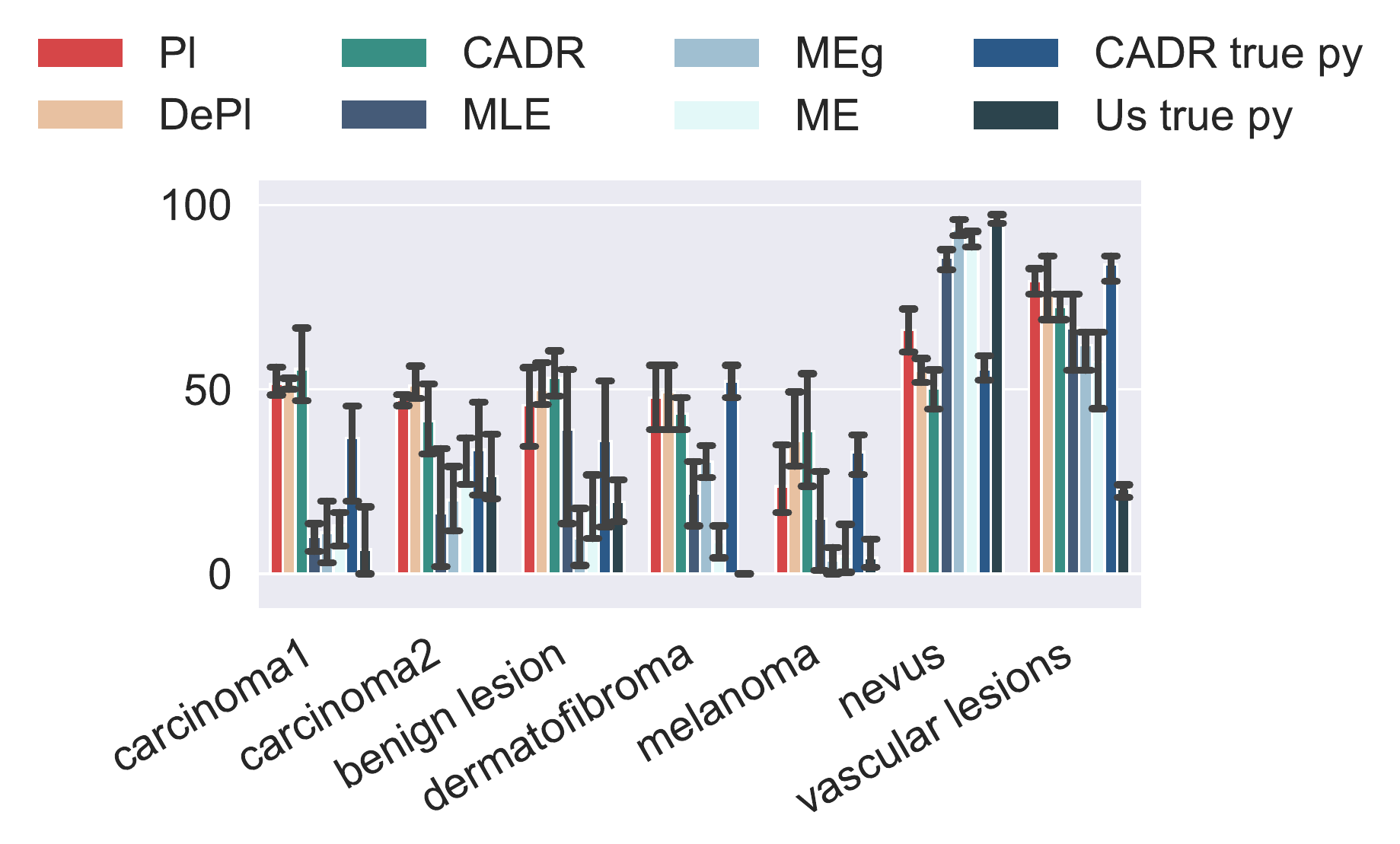}

 \caption{Accuracy per class on dermaMNIST.}
 \label{fig:dermaMNIST}
\end{figure}

\begin{table}[h]
\caption{Accuracy and Loss in the test set on dermaMNIST.} 
\label{table:dermaMNIST_appendix}
\vskip 0.15in
\begin{center}
\begin{small}
\begin{sc}
\begin{tabular}{ccccc}
\toprule
\textbf{Method}   &\textbf{Loss} & \textbf{Accuracy} & \textbf{Accuracy Naevi}&\textbf{MSE \textbf{$\phi_{\textrm{nevi}}$}}   \\
\midrule
Pl & $1.34 \pm 0.16$ & $57.72 \pm 1.95$  & $66.14\pm 5.86$ &  0.80\\ 
\midrule
CADR & $1.42 \pm 0.060$ & $49.36 \pm 1.91$ & $50.41 \pm 5.38$ &$0.77 \pm 0.02$\\
MLE (Ours) & $0.993 \pm 0.020$ & $66.4 \pm 0.81$& $91.16\pm2.26$ &$0.34 \pm 0.03$ \\
MEg (Ours) & $1.19 \pm 0.148$ & $66.65 \pm 1.76$ &  $93.54\pm2.30$ & $0.42 \pm 0.08$\\
ME (Ours) &  $1.24 \pm 0.087$ & $65.8 \pm 0.78$ &  $85.91\pm 3.05$ &$0.38 \pm 0.15$\\
\midrule
CADR ($p(y)$ known) & $1.57 \pm 0.12$ & $49.44 \pm 3.27$ & $55.40\pm3.39$ &\\
Algorithm \ref{alg:second} ($p(y)$ known) & $0.943 \pm 0.029$ & $68.83 \pm 0.26$ & $96.12 \pm 1.20$& \\
\bottomrule
\end{tabular}
\end{sc}
\end{small}
\end{center}
\vskip -0.1in
\end{table}


\begin{figure}[H]    
   \centering \includegraphics[width=0.4\textwidth]{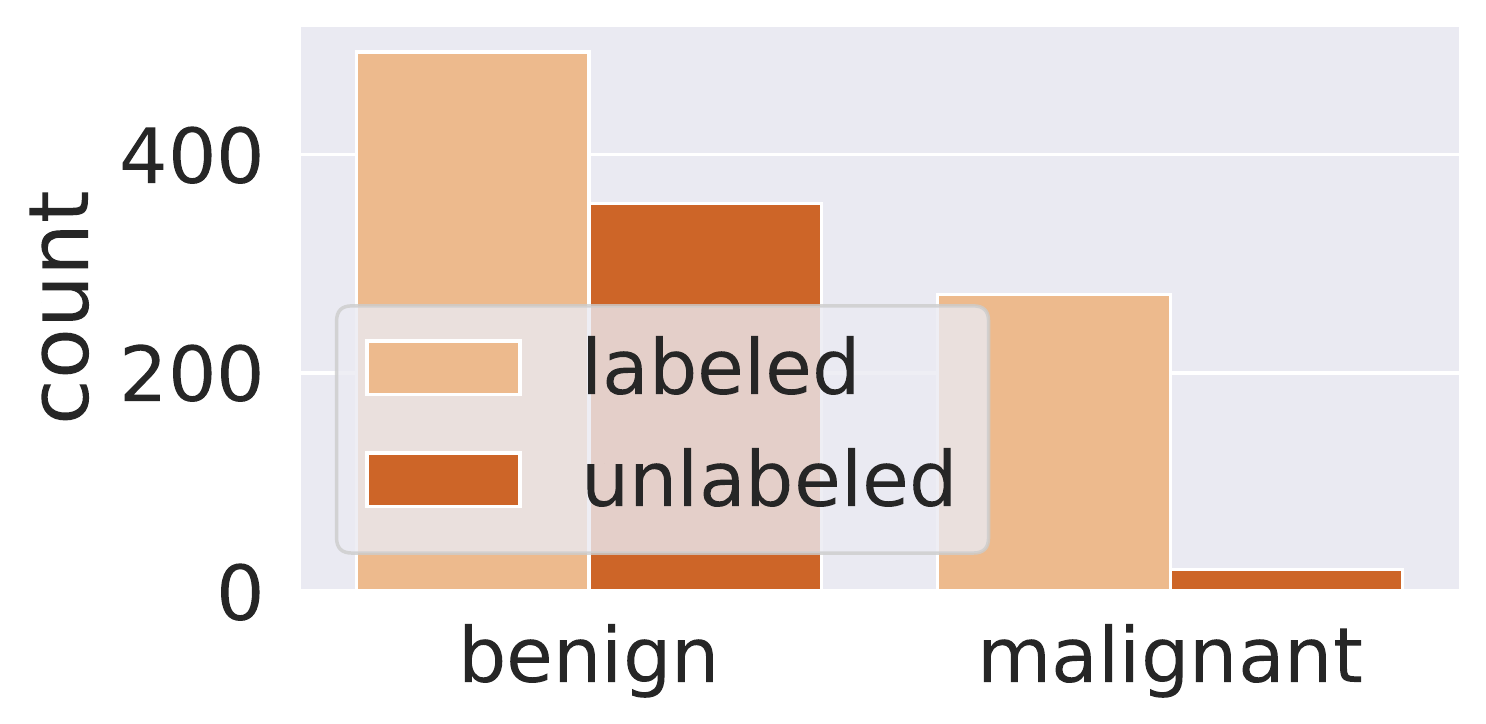}
    \caption{MNAR labels on noduleMNIST (see Section \ref{sec:medMNIST} for an explanation of the missing-data scenario)}
    \label{fig:noduleMNIST}
\end{figure}

\begin{table}
\caption{Accuracy and Loss in the test set on noduleMNIST3D}
\label{tab:noduleMNIST}
\vskip 0.1in
\begin{center}
\begin{small}
\begin{sc}
\begin{tabular}{ccccc}
\toprule
\textbf{Method} & \textbf{Loss}   & \textbf{Specificity (begnin)} & \textbf{Specificy (malign)} & MSE $\phi$   \\
\midrule
PL & 0.389 & 91.06 & 54.69 & 0.0627
\\
\midrule
CADR & 0.48 & 80.08 & 81.25 & 0.0143
\\
MLE (Ours) & 0.359 & 87.80 & 71.88 & 0.0001 \\
MEg (Ours) & 0.353 & 83.74 & 76.56 & 0.0199\\
ME (Ours) & 0.355 & 86.99 & 78.13 & 0.0002 \\
\bottomrule
\end{tabular}
\end{sc}
\end{small}
\end{center}
\vskip -0.1in
\end{table}

\end{document}